\documentclass[twoside,11pt]{article}
\usepackage{jmlr2e}
\usepackage[T1]{fontenc}
\usepackage{verbatim}
\usepackage{amsmath}
\usepackage{kantlipsum}
\allowdisplaybreaks
\usepackage{amssymb}
\usepackage{stackrel}
\usepackage{verbatim}
\usepackage{url}
\usepackage{hyperref}
\usepackage{hyperref}
\hypersetup{
    colorlinks=true,
    linkcolor=black,
    filecolor=magenta,      
    urlcolor=cyan,
    citecolor=orange	
}
\usepackage{graphicx}
\usepackage{enumerate}
\usepackage{subfig}
\usepackage[export]{adjustbox}

\newcommand\numberthis{\addtocounter{equation}{1}\tag{\theequation}}

\newcommand{\lp}{\left(}
\newcommand{\rp}{\right)}

\newcommand{\mbf}{\mathbf}
\newcommand{\mc}{\mathcal}

\newcommand{\mbb}{\mathbb}
\newcommand{\bds}{\boldsymbol}
\usepackage[dvipsnames]{xcolor}
\usepackage[normalem]{ulem}

\newcommand{\ads}[1]{\textcolor{magenta}{\textbf{ADS:} #1}}

\newcommand\redout{\bgroup\markoverwith{\textcolor{red}{\rule[.5ex]{2pt}{0.4pt}}}\ULon}

\DeclareMathOperator*{\argmax}{arg\,max}
\DeclareMathOperator*{\argmin}{arg\,min}

\newcommand*{\figuretitle}[1]{%
    {\centering
    \textbf{#1}
    \par\medskip}
}
\newcommand*{\QEDB}{\null\nobreak\hfill\ensuremath{\square}}%
\newcommand{\model}{\mathrm{M}}
\renewcommand{\P}{\mbb{P}}
\newcommand{\E}{\mbb{E}}

\newcommand{\KL}{\bds{D}_{\text{KL}}}

\newcommand{\tpath}{\mathrm{path}}

\newcommand{\bx}{\mbf{x}}
\newcommand{\bX}{\mbf{X}}
\newcommand{\bnX}{\mbf{Y}} 
\newcommand{\bnx}{\mbf{y}}

\newcommand{\np}{\p_{\dagger}}
\newcommand{\T}{\mathrm{T}}
\newcommand{\p}{\mathrm{p}}
\newcommand{\TCL}{\T^\mathrm{CL}}
\newcommand{\TCLn}{\T^\text{CL}_{\dagger}}
\newcommand{\n}{n}
\newcommand{\nn}{n_{\dagger}}

\newcommand{\eps}{\epsilon}

\newcommand{\treedistr}{\mc{P}_{\T}}

\newcommand{\Vast}{\bBigg@{5}}

\newcommand{\DKL}{D_{\text{KL}}}

\newcommand{\Q}{\mathrm{Q}}
\newcommand{\PG}{\mathrm{P}}
\newcommand{\Ith}{\mathbf{I}^o}
\newcommand{\MI}{\hat{I}}
\newcommand{\fset}{\mc{E}\mc{V}^2}

\newcommand{\evenint}{\mc{I}_{p-1}}
\newcommand{\indep}{\perp \!\!\! \perp}

\usepackage{algorithm}
\usepackage[noend]{algpseudocode}
\usepackage{float}
\usepackage{bbm}
\usepackage{xcolor}
\usepackage{dsfont}
\usepackage{lastpage}

\usepackage{enumitem}
\setitemize{noitemsep}

\ShortHeadings{Optimal Rates for Learning Hidden Tree Structures}{Nikolakakis, Kalogerias and Sarwate}
\firstpageno{1}

\newtheorem{assumption}{Assumption}

\begin{document}

\title{Optimal Rates for Learning Hidden Tree Structures
}


\author{\name Konstantinos E. Nikolakakis \email k.nikolakakis@rutgers.edu \\
       \addr Department of Electrical \& Computer Engineering\\
       Rutgers, The State University of New Jersey \\
	   94 Brett Road, Piscataway, NJ 08854, USA
       \AND
       \name Dionysios S. Kalogerias \email dionysis@msu.edu \\
       \addr  Department of Electrical \& Computer Engineering\\
       Michigan State University\\
       428 S. Shaw Lane, MI 48824, USA
       \AND 
       \name Anand D. Sarwate \email anand.sarwate@rutgers.edu \\
       \addr Department of Electrical \& Computer Engineering\\
       Rutgers, The State University of New Jersey \\
	   94 Brett Road, Piscataway, NJ 08854, USA}


\maketitle

\begin{abstract}%
We provide high probability finite sample complexity guarantees for hidden non-parametric structure learning of tree-shaped graphical models, whose hidden and observable nodes are discrete random variables with either finite or countable alphabets. We study a fundamental quantity called the \textit{(noisy) information threshold}, which arises naturally from the error analysis of the Chow-Liu algorithm and, as we discuss, provides explicit necessary and sufficient conditions on sample complexity, by effectively summarizing the difficulty of the tree-structure learning problem. Specifically, we show that the finite sample complexity of the Chow-Liu algorithm for ensuring exact structure recovery from noisy data is inversely proportional to the information threshold squared (provided it is positive), and scales almost logarithmically relative to the number of nodes over a given probability of failure. Conversely, we show that, if the number of samples is less than an absolute constant times the inverse of information threshold squared, then no algorithm can recover the hidden tree structure with probability greater than one half. As a consequence, our upper and lower bounds match with respect to the information threshold, indicating that it is a fundamental quantity for the problem of learning hidden tree-structured models. Further, the Chow-Liu algorithm with noisy data as input achieves the optimal rate with respect to the information threshold. Lastly, as a byproduct of our analysis, we resolve the problem of tree structure learning in the presence of \textit{non-identically} distributed observation noise, providing conditions for convergence of the Chow-Liu algorithm under this setting, as well.
\end{abstract}


\section{Introduction}
%
%
%
%

Graphical models are a widely-used and powerful tool for analyzing high-dimensional structured data \citep{lauritzen1996graphical,koller2009probabilistic}. In those models, variables are represented as nodes of a graph, the edges of which indicate conditional dependencies among the corresponding nodes. In this paper, we consider the problem of \textit{learning acyclic undirected graphs} or \textit{tree-structured Markov Random Fields (MRFs)}. In particular, we provide optimal rates for the problem of learning tree structures under minimal assumptions, when noisy data are available. First, we analyze the sample complexity of the well-known \textit{Chow-Liu (CL) Algorithm} \citep{chow1968approximating}, which, given a data set of samples, returns an estimate of the original tree. In the absence of noise the importance of the CL algorithm  stems from its efficiency, its low computational complexity, and also its optimality in terms of sample complexity (matching information-theoretic limits) for a variety of statistical settings, e.g., such as those involving  Gaussian data \citep{tan2009learning}, binary data \citep{bresler2020learning}, as well as for non-parametric models with discrete alphabets \citep{tan2011large,tan2011learning}. Second, the CL algorithm with noisy data as input maintains its computational efficiency, and as we show, achieves optimal sample complexity rates for certain cases of hidden tree-structured models

 In many applications, the underlying physical or artificial phenomenon may be well-modeled by a (tree-structured) MRF, but the data acquisition device or sensor may itself introduce noise. We wish to understand how sensitive the performance of an algorithm is to noisy inputs. The problem of learning hidden structures receives significant attention the last few years. Prior or concurrent works include partial~\cite{tandon2020exact,katiyar2020robust,tandon2021sga,casanellas2021robust,katiyar2021robust}, as well as exact structure recovery~\citep{goel2019learning,nikolakakis2019learning,nikolakakis2021predictive}. Specifically, ~\cite{goel2019learning,nikolakakis2019learning} study the impact of corrupted observations on the exact recostruction for both binary and Gaussian models.~\cite{goel2019learning} extend the Interaction Screening Objective~\citep{vuffray2016interaction} for the case of Ising models, while our prior work~\citep{nikolakakis2019learning} analyzes the performance of the CL algorithm for both noisy Ising and Gaussian models. 
 Based on this prior work, it is tempting to think that one can successfully perform structure recovery by running the standard CL algorithm \textit{directly} on noisy data, for general hidden tree-models and minimal assumptions. 


Our contribution centers on a rigorous characterization of
a fundamental model-dependent statistic, which we call the \textit{information threshold} (denoted by $\Ith$ and $\Ith_{\dagger}$ for the noiseless and noisy settings, respectively). 
For the problem of learning tree-structures from noiseless data, the information threshold has already appeared in fundamental prior work in the area (e.g., see prior work by \cite{tanlargeIEEE} and \cite{tan2011largethesis}). 
Precisely, Theorem 6 by~\citet{tanlargeIEEE} states that if the information threshold is strictly positive, then the probability of incorrect structure recovery by the CL algorithm decays exponentially with respect to the number of samples. 
Specifically,~\cite{tan2011large,tanlargeIEEE} show that the approximation of the error exponent is consistent with the true error exponent when $\Ith\to 0$. However, no conclusions can be derived based on results from prior works regarding the rate of the sample complexity with respect to $\Ith$ when it is bounded away from zero, or when the data are noisy.

In this paper, we provide finite sample complexity bounds for exact \textit{hidden} tree-structure learning.  Although it is true that the number of samples required for successful recovery scales logarithmically with respect to the number of nodes $p$ (in the asymptotic sense), the sample complexity can vary significantly for different tree-structured distributions or noise models (for fixed $p$), and fixed probability of failure $\delta$. The discussion above naturally raises the following basic question: \textit{Is there any essential statistic (e.g., the information threshold) that determines the sample complexity of the hidden tree-structure learning problem?} 
We show that, indeed, the noisy information threshold constitutes such a representative statistic; in fact, it provides explicit \textit{necessary and sufficient} conditions on sample complexity, by effectively and completely summarizing the difficulty of the hidden tree-structure learning problem. First, for fixed positive values of probability of success and (noisy) information threshold (namely $\Ith_{\dagger}$), we derive finite complexity bounds for exact structure recovery, by exploiting the Chow-Liu algorithm with noisy observations as its inputs. Second, we show that our sample complexity bounds achieve the optimal rate with respect to $\Ith_{\dagger}$.

Our contributions are summarized as follows.

\begin{figure}[t!]
\centering
  \includegraphics[width=210pt]{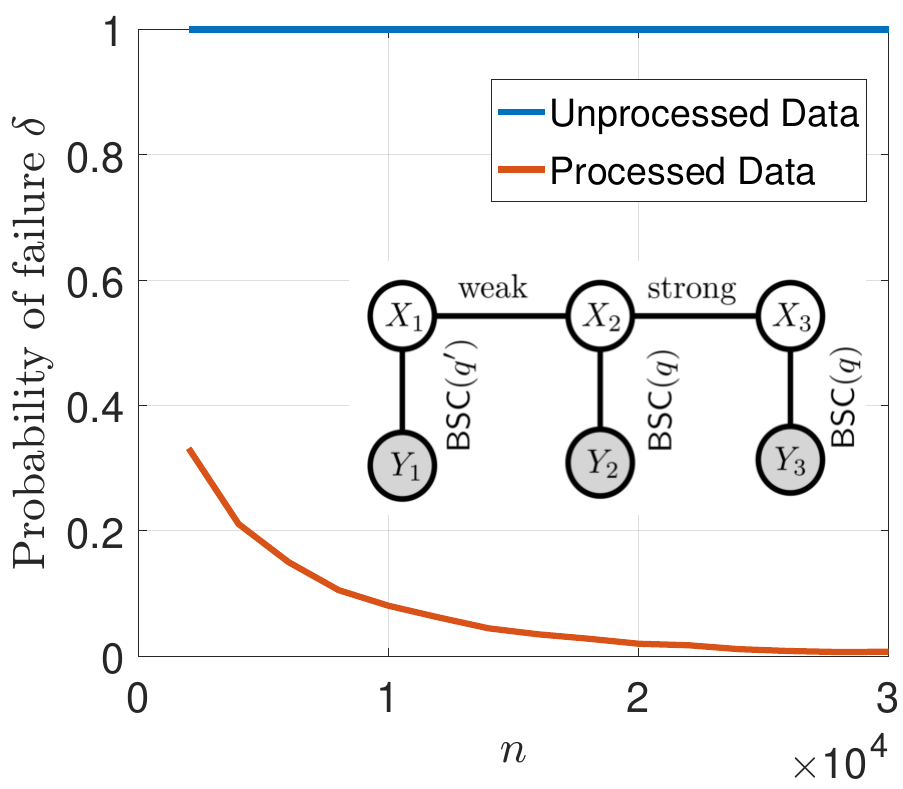}
  \caption{$\delta\equiv\hat{\P}_n\big( \TCL_{\dagger}\neq\T \big)$ for $3$-node hidden structure whose $\Ith_{\dagger}<0$ (non-identically distributed noise). "Weak" and "Strong" refer to correlations between hidden variables: $|\E [X_1 X_2]|< |\E [X_2 X_3]|$.}
  \label{fig:fintro}
 \end{figure}
\setitemize{leftmargin=10pt}
 \setlist{nolistsep}
\begin{itemize}[leftmargin=*,noitemsep]
    \item We provide the first finite sample complexity results on \textit{non-parametric} hidden tree-structure learning, where the distributions of the hidden and observable layers are \textit{unknown} and the mappings between the hidden and observed variables are general. In fact, as with the noiseless case, we show that, as long as $\Ith_{\dagger}>0$, the sample complexity of the CL algorithm with respect to $p/\delta$ scales as $\mathcal{O}\big(\log^{1+\zeta} (p/\delta)\big) $, for any $\zeta>0$.
    With respect to the associated noisy information threshold, the same sample complexity is of the order of $\mc{O}(1/(\Ith_{\dagger})^{2(1+\zeta)})$, for any $\zeta>0$ (Theorem \ref{thm:sufficient_noisy}).
     \item For an absolute constant $C>0$, we show that if the number of samples is less than $C/(\Ith_{\dagger})^2$, then no algorithm can recover the hidden tree structure with probability greater than one half (Theorem \ref{thm:converse_noisy}). The rate of of upper and lower bounds with respect to $\Ith_{\dagger}$ are essentially identical, and CL algorithm in the noisy setting achieves the optimal rate with respect to $\Ith_{\dagger}$.
    \item Additionally, we show that, if $\Ith_{\dagger}\leq 0$, then structure recovery from \textit{raw} data is not possible. As an example of the case $\Ith_{\dagger}\leq 0$ consider the model in Figure \ref{fig:fintro}, for which structure recovery from unprocessed data fails. In some cases where $\Ith_{\dagger}< 0$, by introducing suitable \textit{processing} and enforcing appropriate conditions on the hidden model (Definition \ref{ordering_processed}), we show that the CL algorithm can be made convergent (with the aforementioned sample complexity).
    Such conditions can be satisfied for a variety of interesting observation models.
    Essentially, our results confirm that the CL algorithm is an effective \textit{universal estimator} for tree-structure learning, on the basis of noisy data. We lastly explicitly illustrate how our results capture certain interesting scenarios involving \textit{generalized} $M$-ary erasure and symmetric channels (i.e., observation noise models). Our framework extends and unifies recent results proved earlier for the binary case \citep{bresler2020learning,nikolakakis2019learning}, by considering \textit{non-identically distributed noise} as well. 
    Note that structure learning in the presence of non-identically distributed noise remains open for general graph structures; see prior work by~\citet[Section 6]{goel2019learning}. 
\end{itemize}

\subsection{Related Work}

Structure learning from node observations is a fundamental and well-studied problem in the context of graphical models. For general graph structures, the complexity of the problem has been studied by~\citet{karger2001learning}. Under the assumption of bounded degree the problem becomes tractable, leading to a large body of work in the last decade~\citep{vuffray2016interaction,lee2007efficient,bresler2008reconstruction,ravikumar2010high,bresler2014hardness,bresler2015efficiently,vuffray2019efficient}. These approaches employ greedy algorithms, $l_1$ regularization methods, or other optimization techniques. The sample complexity of each approach is evaluated based on information theoretic bounds~\citep{santhanam2012information,tandon2014information,hamilton2017information}.

For acyclic graphs, the CL algorithm is computationally efficient and it has been shown to be optimal in terms of sample complexity with respect to the number of variables. The error analysis and convergence rates for trees and forests were studied by~\citet{tan2009learning,tan2011large,tan2011learning},~\citet{liu2011forest} and~\citet{bresler2020learning}. In particular, for finite alphabets the number of samples needed by the CL algorithm is logarithmic with respect to the ratio $p/\delta$ for tree structures and polylogarithmic in the case of forests: $\mc{O}\big(\log^{1+\zeta} (p/\delta)\big)$, for all $\zeta>0$~\cite[Theorem 5]{tan2011learning}). Comparatively, our results extend to countable alphabets and noisy observations (hidden models). Our sample complexity bounds are consistent with prior work: the order is poly-logarithmic but remains arbitrarily close to logarithmic ($\mc{O}\big(\log^{1+\zeta} (p/\delta)\big)$, for all $\zeta>0$). In the special case of Ising models, hidden structures were considered by~\citet{goel2019learning} and our previous work~\citep{nikolakakis2019learning,nikolakakis2021predictive}. Goel et al.~\citep{goel2019learning} consider bounded degree graphs and require the noise model (or an approximation of it) to be known with the noise i.i.d. on each variable (see their Section 6). Our results apply more generally to settings with \textit{non-identically distributed noise} (Section \ref{M-ary_arasure}), but are restricted to the tree structure model assumption. Prior work by~\citep{nikolakakis2019learning,nikolakakis2021predictive} solves the problem of hidden tree structure learning for (binary) Ising models with (modulo-2) additive noise, while in this paper we generalize prior sample complexity bounds~\citep{nikolakakis2019learning,nikolakakis2021predictive} to finite and countable alphabets, encompassing non-parametric tree structures and general noise models. Finally,~\cite{tandon2020exact,katiyar2020robust,tandon2021sga,casanellas2021robust,katiyar2021robust} consider the problem of learning hidden tree-structured models within an equivalence class. In contrast with these papers, we study the exact structure recovery instead of that of the partial structure learning.

\subsection{Applications}

Noise-corrupted structured data appear in various applications for major branches of science as physics, computer science, biology and finance. The assumption of discrete and (arbitrarily) large alphabets is suitable for many of these scenarios. Our work is also motivated by a number of emerging applications in machine learning: differential privacy, distributed learning, and adversarial corruption. For example in \emph{local differential privacy}~\citep{Warner1965,EGS03,KLNRS08,KLNRS11,Duchi2013,Guha2017}, one approach perturbs each data sample at the time it is collected: our results therefore characterize the increase in sample complexity for successful structure recovery under a given privacy guarantee. In \emph{distributed learning} problems, communication constraints require data to be quantized, resulting in quantization noise on the samples~(c.f.~\citep{tavassolipour2018learning}). Finally, we might have an \emph{adversarial attack} on the training data: an adversary could corrupt the data-set and make structure learning infeasible. As our example shows, non-i.i.d. corruption can make the information threshold negative, and the CL algorithm will fail without appropriate processing of the mutual information estimates. We continue by introducing the notation that we use in our paper.

\subsection{Notation}
 We denote vectors or tuples by using boldface and we reserve calligraphic face for sets. For an integer $M$, let $[M] \triangleq \{1,2,\ldots M \}$, and let $[M]^2\equiv [M]\times [M]$ denote the corresponding Cartesian product. For an odd natural number $p$, we denote the set of even natural numbes up to $p-1$ as $\evenint \triangleq\{2,4,\ldots, p-1\}$. The indicator function of a set $A$ is denoted by $\mathds{1}_A$. In our graphica models, the cardinality of the set of nodes $\mc{V}$ is assumed to be equal to $p$, $|\mc{V}|= p$. The node variables of the tree are denoted by $\bX=\lp X_1,X_2,\ldots,X_p\rp$. If $\bX$ has a finite alphabet we write $\bX\in [M]^p$, and $\bX\in\mc{X}^p$ for countable size alphabets.  The probability mass function of $\bX$ is denoted by $\p(\cdot)$ and the joint distribution of the pair $X_i,X_j$ by $\p_{i,j} (\cdot,\cdot)$, for any $i,j\in\mc{V}$. $\T=(\mc{V},\mc{E})$ is a tree (acyclic graph) with set of nodes and edges $\mc{V}$ and $\mc{E}$ respectively. We denote the unique set of edges in the path of two nodes $u,\bar{u}\in\mc{V}$ of $\T$ as $\tpath_{\T} (u,\bar{u})$. The neighborhood of a node $\nu\in\mc{V}$ is denoted by $\mc{N}_{\T}(\nu)$. The distribution $\hat{\p}(\cdot)$ is the estimator of a distribution $\p\lp\cdot \rp$. The symbol $\bX^{1:n}$ denotes the dataset of $n$ i.i.d samples of $\bX$. The estimated mutual information of a pair $X_i,X_j$ is $\MI \lp X_i;X_j \rp$ and the resulting tree structure of the CL algorithm is denoted by $\TCL$. The noisy observable is denoted by $\bnX$ and we use the symbol $\dagger$ to distinguish between noiseless quantities and the corresponding noisy quantities, for instance $\TCL_{\dagger}$ is the estimated tree from $\bnX^{1:n}$.

\section{Model and Problem Statement}

First, we provide a complete description of our model including definitions, properties, and assumptions on the underlying distributions.   

\subsection{Tree-Structured and Hidden Tree-Structured Models}\label{model_formulation}

 We consider graphical models over $p$ nodes with variables $\{X_1,X_2,\ldots,X_p\}\equiv\bX$ and finite or countable alphabet $\mc{X}^p$. We assume that the distribution $\p(\cdot)$ of $\bX$ is given by a tree-structured \textit{Markov Random Field, (MRF)}. Any distribution $\p(\cdot)$ which is Markov with respect to a tree $\T=(\mc{V},\mc{E})$ factorizes as~\citep{lauritzen1996graphical}
	\begin{align}\label{eq:tree_shaped}
	\!\!\!\p(\bx)= \prod_{i\in V} \p\left(x_{i}\right) 
	\prod_{ (i,j) \in \mc{E}} \frac{\p(x_{i},x_{j})}{\p(x_{i}) \p(x_{j})},\:\bx\in\mc{X};
	\end{align} 
we call such distributions $\p (\cdot)$ tree-structured. 

The noisy node variables $\bnX=\{Y_1,Y_2,\ldots,Y_p\}$ are generated by a \emph{randomized mapping} (noisy channel) $\mc{F}:\mc{X}^p \to \mc{Y}^p$. We restrict attention here to mappings that act on each component (not necessarily independently): $\mc{F}(\cdot)=\{F_1(\cdot),F_2(\cdot),\ldots,F_p(\cdot) \}$ and each $F_i(\cdot) : \mc{X} \to \mc{Y}$ so $Y_i = F_i (X_i)$ for all $i\in\mc{V}$. Let $\P \lp Y_i=y_i| X_i=x_i \rp$ be the transition kernel associated with the randomized mapping $F_i(\cdot)$, then the distribution of the output is given by $\p_{\dagger}(y_i)= \sum_{x_i\in\mc{X}}\P(Y_i=y_i|X_i=x_ i)\p(x_i)$ for $y_i\in\mc{Y}$. Note that while the distribution of $\bX$ is tree-structured, the distribution of $\bnX$ \textit{does not factorize} according to any tree. In general the Markov random field of $\bnX$ is a complete graph, which makes learning the hidden structure non-trivial~\citep{nikolakakis2019learning,nikolakakis2021predictive,yeung2019information}.



Given $n$ i.i.d observations $\bX^{1:n} \sim \p(\cdot)$, our goal is to learn the underlying structure $\T$. To do this, we use a plug-in estimate of the mutual information $I(X_i;X_j)$ between pairs of variables. In similar fashion, when noise-corrupted observations $\bnX^{1:n}$ are available, we aim to learn the hidden tree structure $\T$ of $\bX$, by estimating the mutual information $I(Y_i;Y_j)$ between pairs of observable variables. Unfortunately, the plug-in mutual information estimate $\hat{I}(X_i;X_j)$ may converge slowly to $I(X_i;X_j)$ in certain cases with countable alphabets~\cite[Corollary 5]{antos2001convergence},~\citep{silva2018shannon}. To avoid such ill-conditioned cases, we make the following assumption.

\begin{assumption}\label{antos_ass}
 For some $c>1$ there exist $c_1,c_2>0$ such that $ c_1/k^c\leq \p_i (k)\leq c_2 /k^c$, for $k\in \mc{X}$, and $ c_1/(k\ell)^c\leq \p_{i,j} (k,\ell)\leq c_2 /(k\ell)^c$, for $k,\ell\in \mc{X}^2$ and for all $i,j\in\mc{V}$. That is, the tuple $\{c,c_1,c_2\}$ satisfies the assumption for all marginal and pairwise joint distributions.
\end{assumption}

Assumption \ref{antos_ass} holds trivially for all finite and (fixed) alphabets, where the constants $c_1$ and $c_2$ depend on the minimum probability and the size of alphabet.
The next assumption guarantees there is a unique tree structure $\T$ with exactly $p$ nodes.

\begin{assumption}\label{ass:unique_tree}
 $\T$ is connected; $I(X_i;X_j)>0$ for all $i,j\in\mc{V}$ and the distribution $\p (\cdot)$ of $\bX$ is not degenerate.
\end{assumption}

Assumption \ref{ass:unique_tree} guarantees convergence for the CL algorithm in the absence of noise; $\TCL\to\T$. For a fixed tree $\T$, we use the notation $\treedistr$ to denote the set of structured distributions that factorizes according to a fixed $\T$. The set of all trees on $p$ nodes is denoted by $\mc{T}$, and we call $\mc{P}_{\mc{T}}$ the set of all tree-structured distributions that factorize according to \eqref{eq:tree_shaped} for some $\T\in\mc{T}$.

\subsection{Chow-Liu Algorithm}
\begin{algorithm}[t]
	\caption{Chow-Liu (CL) \label{alg:Chow-Liu}}
	\begin{algorithmic}[1]
    \Require Data set $\mc{D}=\boldsymbol{Z}  \in\mc{Z}^{|\mc{V}|\times n}$
         \State $\hat{\p}_{i,j}(\ell,m)=\frac{1}{n}\sum^{n}_{k=1}\mathds{1}_{\{Z_{i,k}=\ell,Z_{j,k}=m\}},$  $\forall i,j \in \mc{V}$ 
         \State $\MI(Z_i;Z_j)=\sum_{\ell,m} \hat{\p}_{i,j}(\ell,m)\log_2 \frac{\hat{\p}_{i,j}(\ell,m)}{\hat{\p}_{i}(\ell)\hat{\p}_{j}(m)}$
\State  $\T^{\text{estimate}} \gets$ MST$\lp  \{\MI(Z_i;Z_j):i,j\in\mc{V} \} \rp$
    \end{algorithmic}
\end{algorithm}
We consider the classical version of the algorithm~\citep{chow1968approximating}, due to the non-parametric nature of our model. Given $n$ i.i.d samples of the node variables, we first find the estimates of the pairwise joint distributions and then we evaluate the plug-in mutual information estimates. Finally, a Maximum Spanning Tree (MST) algorithm (for instance Kruskal's or Prim's algorithm) returns the estimated tree. For the rest of the paper, we refer to Algorithm \ref{alg:Chow-Liu} by explicitly mentioning the input data set; if $\mc{D}=\bX^{1:n}$, then the input consist of noiseless data and we consider $\boldsymbol{Z}\equiv \bX$ (see Algorithm \ref{alg:Chow-Liu}), furthermore the estimated structure $\T^{\text{estimate}}$ is denoted by $\TCL$. Equivalently, if $\mc{D}=\bnX^{1:n}$, then the input consists of noisy data, $\boldsymbol{Z}\equiv\bnX$, and the estimated structure $\T^{\text{estimate}}$ is named as $\TCLn$. We compute $\TCL$ and $\TCLn$ by running the MST algorithm on the edge weights $\{\MI(X_i;X_j):i,j\in\mc{V} \}$ and $\{\MI(Y_i;Y_j):i,j\in\mc{V} \}$ respectively. Note that the estimates $\TCL$ and $\TCLn$ depend on the value $n$, but for brevity, we write $\TCL\to\T$ and $\TCLn\to\T$ instead of  $\lim_{n\to\infty}\TCL= \T$  and $\lim_{n\to\infty}\TCLn = \T$ respectively. In the next section we show that Chow-Liu algorithm is equivalent to a maximum likelihood estimator:  Algorithm \ref{alg:Chow-Liu} returns the MLE-tree of the projected distribution of the observables onto the space of tree-structured distributions.

\subsection{Projected Maximum Likelihood Estimate }
We now show the connection between the maximum likelihood estimate tree and the output $\TCLn$ of Chow-Liu algorithm with input noisy data. As a first step, we state the following Lemma, that generalizes the result on Ising model distribution by \citet[Lemma 1]{bresler2020learning} to general distributions.\footnote{We are unaware of a prior proof of the following result and so we provide it for completeness.} 

\begin{lemma}\label{lemma:projection}
Let $\PG (\cdot)$ be a distribution that factorizes according a graph $\mathrm{G}=(\mc{V},\mc{E}_{\mathrm{G}})$, and $\Q(\cdot)\in\mathcal{P}_{\mc{T}}$ a tree structured distribution with $\T=(\mc{V},\mc{E})$, such that $P(\cdot)$ is absolutely continuous with respect to $\Q(\cdot)$. Let $I_{\PG}(\cdot;\cdot)$ be the mutual information with respect to $\PG(\cdot)$ and consider the projection of $\PG(\cdot)$ onto the space of tree-structured distributions 
\begin{align}
    \p^\mc{T} (\bnx)\triangleq \argmin_{\bar{\Q}(\cdot)\in\mc{P}_{\mc{T}}} \DKL(\PG (\bnx) ||\bar{\Q} (\bnx)).
\end{align} Then
\begin{align}
     \DKL(\PG(\bnx) || \p^\mc{T} (\bnx)) =  -H(\PG (\bnx)) + \sum_{i\in V} H(\PG (y_i))  -	\sum_{ (i,j) \in \mc{E}}I_{\PG}(Y_i;Y_j)  .\label{eq:PMLE2}
\end{align}
\end{lemma}\noindent 
\begin{proof}
It is true that
\begin{align*}
    &\DKL(\PG (\bnx) ||\Q (\bnx))\\&=-H(\PG (\bnx)) -\E_{\PG (\bnx)} [\log \Q (\bnX)]\numberthis \label{eq:PMLE1}\\
    &=-H(\PG (\bnx)) -\E_{\PG (\bnx)} \left[ \log \prod_{i\in V} \Q\left(Y_{i}\right) +\log
	\prod_{ (i,j) \in \mc{E}} \frac{\Q(Y_{i},Y_{j})}{\Q(Y_{i}) \Q(Y_{j})}\right]\\
	&=-H(\PG (\bnx)) -\sum_{i\in V}\E_{\PG (\bnx)} \left[ \log  \Q\left(Y_{i}\right)\right] -	\sum_{ (i,j) \in \mc{E}}\E_{\PG (\bnx)}\left[ \log
 \frac{\Q(Y_{i},Y_{j})}{\Q(Y_{i}) \Q(Y_{j})}\right]\\
 &=  -H(\PG (\bnx))  -\sum_{i\in V}\E_{\PG (\bnx)} \left[ \log  \Q\left(Y_{i}\right)\right]    -\sum_{(i,j) \in \mc{E}}\E_{\PG (\bnx)} \left[ \log  \frac{\PG(Y_{i})\PG(Y_{j})}{\Q(Y_{i})\Q(Y_{j})}\right] \\
 &\quad +\sum_{(i,j) \in \mc{E}}\E_{\PG (\bnx)} \left[ \log  \PG(Y_{i})\PG(Y_{j})\right] -	\sum_{ (i,j) \in \mc{E}}\E_{\PG (\bnx)}\left[ \log
 \Q(Y_{i},Y_{j})\right]\\
  &=  -H(\PG (\bnx))  -\sum_{i\in V}\E_{\PG (\bnx)} \left[ \log  \Q\left(Y_{i}\right)\right]    -\sum_{(i,j) \in \mc{E}}\DKL\lp\PG(y_{i}) || \Q(y_{i}) \rp+\DKL\lp\PG(y_{j}) || \Q(y_{j}) \rp \\
    &\quad +
    \sum_{(i,j) \in \mc{E}}\E_{\PG (\bnx)} \left[ \log  \PG(Y_{i})\PG(Y_{j})\right]-	\sum_{ (i,j) \in \mc{E}}\E_{\PG (\bnx)}\left[ \log
 \Q(Y_{i},Y_{j})\right] \\
 &= -H(\PG (\bnx))  -\sum_{i\in V}\E_{\PG (\bnx)} \left[ \log  \Q\left(Y_{i}\right)\right]+	\sum_{ (i,j) \in \mc{E}} H(\PG (y_i,y_j)) - H(\PG (y_i))-H(\PG (y_j)) \\
    &\quad + 	
    \sum_{ (i,j) \in \mc{E}} \DKL \lp \PG (y_i,y_j)||  \Q(y_{i},y_{j})\rp
    -\sum_{(i,j) \in \mc{E}}\DKL\lp\PG(y_{i}) || \Q(y_{i}) \rp+\DKL\lp\PG(y_{j}) || \Q(y_{j}) \rp\\
 &= -H(\PG (\bnx))   +\sum_{i\in\mc{V}} \DKL(\PG (y_{i}) ||\Q (y_{i}) + \sum_{i\in\mc{V}} H(\PG (y_{i}) -	\sum_{ (i,j) \in \mc{E}} I_{\PG}(Y_i;Y_j) \\
    &\quad + 	\sum_{ (i,j) \in \mc{E}}  \DKL \lp \PG (y_i,y_j)||  \Q(y_{i},y_{j})\rp -\sum_{(i,j) \in \mc{E}}\DKL\lp\PG(y_{i}) || \Q(y_{i}) \rp+\DKL\lp\PG(y_{j}) || \Q(y_{j}) \rp. \numberthis \label{eq:equiv_problem_1stpart}
\end{align*} 
At this point the following notation is necessary (see \citep{koski2010lectures}):  Let $(i_r)^{p}_{r=1}$ be an arbitrary permutation of $\bds{\ell}=\{1,2,\ldots,p\}$. Then, the set of edges $\mc{E}$ for any tree can be written  as\begin{align}\label{eq:dependence_structure}
\mc{E}=\lp i_r,j_r \rp^p _{r=2},\quad \text{ and }j_1=\emptyset,\quad j_r \in \{i_1,\ldots,i_{r-1} \}\subset\bds{\ell}.  
\end{align} \eqref{eq:dependence_structure} defines a tree $\T=(\mc{V},\mc{E})$ with root the node $i_1$ (since $j_1=\emptyset$). Then we write \eqref{eq:equiv_problem_1stpart} as
\begin{align*}
    &\DKL(\PG (\bnx) ||\Q (\bnx))\\
 &= -H(\PG (\bnx))  +\sum^p_{r=1} \DKL(\PG (y_{i_r}) ||\Q (y_{i_r})) + \sum^p_{r=1} H(\PG (y_{i_r}))   -	\sum_{ (i,j) \in \mc{E}} I_{\PG}(Y_i;Y_j) \\
 &\quad + 	\sum^p_{ r=2}  \DKL \lp \PG (y_{i_r},y_{j_r} )  || \Q(y_{i_r},y_{j_r})\rp  -\sum^p_{ r=2}\DKL\lp\PG(y_{i_r})||\Q(y_{i_r}) \rp-\sum^p_{r=2} \DKL\lp\PG(y_{j_r})||\Q(y_{j_r}) \rp\\
  &= -H(\PG (\bnx))  + \DKL(\PG (y_{i_1}) ||\Q (y_{i_1})) + \sum^p_{r=1} H(\PG (y_{i_r}))   -	\sum_{ (i,j) \in \mc{E}} I_{\PG}(Y_i;Y_j) \\
 &\quad + 	\sum^p_{ r=2}  \DKL \lp \PG (y_{i_r},y_{j_r} ) || \Q(y_{i_r},y_{j_r})\rp  -\sum^p_{r=2} \DKL\lp\PG(y_{j_r})||\Q(y_{j_r}) \rp\label{eq:equiv_problem_2ndpart}\numberthis
\end{align*}

\noindent Notice that \begin{align}
     \DKL(\PG (y_{i_1}) ||\Q (y_{i_1}))\geq 0 \label{eq:minimization_KL}
\end{align} and \begin{align}
    \sum^p_{ r=2}  \DKL \lp \PG (y_{i_r},y_{j_r} ) ||  \Q(y_{i_r},y_{j_r})\rp  -\sum^p_{r=2} \DKL\lp\PG(y_{j_r}) || \Q(y_{j_r}) \rp \geq 0,\label{eq:minimization_KLsums}
\end{align} because of the monotonicity property of the KL divergence. Thus to minimize \eqref{eq:equiv_problem_2ndpart} we need the equality to hold in \eqref{eq:minimization_KL} and \eqref{eq:minimization_KLsums}. Under the choice of $\Q(\cdot)$ as the projection onto the space of tree structured distributions as $ \p^\mc{T} (\bnx)\triangleq \argmin_{\bar{\Q}(\cdot)\in\mc{P}_{\mc{T}}} \DKL(\PG (\bnx) ||\bar{\Q} (\bnx))$, we get the matching on the distribution of the root $\PG (y_{i_1}) = \Q (y_{i_1})$ and the matching on the pairwise marginal distributions $\PG (y_{i_r},y_{j_r})\equiv  \Q(y_{i_r},y_{j_r})$ for all $r \in \{2,3,\ldots,p\}$, or equivalently $\PG (y_i,y_j)\equiv  \Q(y_{i},y_{j})$ for all $(i,j) \in \mc{E}$ (see \eqref{eq:dependence_structure}). The latter makes the KL divergences equal to zero in \eqref{eq:minimization_KL} and \eqref{eq:minimization_KLsums}, and \eqref{eq:equiv_problem_2ndpart} directly gives \eqref{eq:PMLE2}.
\end{proof}

  The true distribution of the observable layer $\bnX$ and the estimated distribution $\hat{\p} (\cdot)$ 
   from $n$ i.i.d noisy samples $\{\bnX^{(1)},\bnX^{(2)},\ldots,\bnX^{(n)}\}=\mc{D}$ do not factorize according to any tree. Although the noisy data do not come from a tree-structured distribution (in contrast with prior work \citep{tan2011large}), the approach by~\citet[Section III.]{tan2011large} can be extended as follows.  Given $n$ i.i.d noisy samples $\mc{D}$ from a hidden tree-structured model, Chow-Liu algorithm (Algorithm \ref{alg:Chow-Liu}) with input $\mc{D}$ returns the maximum likelihood estimate tree of the projected distribution
\begin{align}
    \hat{\p}^\mc{T} (\bnx)\triangleq \argmin_{\bar{\Q}(\cdot)\in\mc{P}_{\mc{T}}} \DKL(\hat{\p} (\bnx) ||\bar{\Q} (\bnx)).
\end{align} 

\begin{corollary}\label{cor:PMLE} The maximum likelihood estimate tree of $ \hat{\p}^\mc{T} (\bnx)$ is equivalent with the output $\TCLn$ of Chow-Liu algorithm, that is,
\begin{align}
     \hat{\T}_{\textup{PMLE}}\triangleq  \argmax_{\T\in\mc{T}} \log \prod^n_{s=1} \hat{\p}^\mc{T} (\bnX^{(s)})=\argmax_{\T\in\mc{T}}  \sum_{(i,j)\in\mc{E}_{\T}} \hat{I}(Y_i;Y_j) \equiv \TCLn.
\end{align}
\end{corollary} 

\begin{proof}
Starting from the definition of $\hat{\T}_{\text{PMLE}}$
\begin{align*}
   \hat{\T}_{\text{PMLE}}\triangleq \argmax_{\T\in\mc{T}} \log \prod^n_{s=1} \hat{\p}^\mc{T} (\bnX^{(s)})&=\argmax_{\T\in\mc{T}} \sum_{\bnx\in\mc{Y}} \frac{1}{n}\sum^n_{s=1}\mathds{1}_{\bnX^{(s)}=\bnx} \log \hat{\p}^\mc{T} (\bnx)\\ &=\argmax_{\T\in\mc{T}} \sum_{\bnx\in\mc{Y}}  \hat{\p} (\bnx) \log \hat{\p}^\mc{T} (\bnx)
   \\&= \argmax_{\T\in\mc{T}}\E_{\hat{\p} (\bnx)} [\log \hat{\p}^\mc{T} (\bnX)]\numberthis \label{eq:TPMLE1}. 
\end{align*} We apply Lemma \ref{lemma:projection} by replacing  $\PG(\bnx)$ with $\hat{\p}(\bnx)$ and $\Q(\bnx)$ with $ \hat{\p}^\mc{T} (\bnx)$. Then \eqref{eq:TPMLE1} through \eqref{eq:PMLE1} gives \begin{align}
     \hat{\T}_{\text{PMLE}}  = \argmax_{\T\in\mc{T}}\E_{\hat{\p} (\bnx)} [\log \hat{\p}^\mc{T} (\bnX)]= \argmin_{\T\in\mc{T}}  \DKL(\hat{\p}(\bnx) || \hat{\p}^\mc{T} (\bnx)),\label{eq:tpmleKL}
\end{align} and \eqref{eq:PMLE2} is written as \begin{align}
     \DKL(\hat{\p}(\bnx) || \hat{\p}^\mc{T} (\bnx)) =  -H(\hat{\p}(\bnx)) + \sum_{i\in V} H(\hat{\p} (y_i))  -	\sum_{ (i,j) \in \mc{E}}\hat{I}(Y_i;Y_j).\label{eq:KLMI}
\end{align} Finally, \eqref{eq:tpmleKL} and \eqref{eq:KLMI} give $\hat{\T}_{\text{PMLE}}=\argmax_{\T\in\mc{T}}  \sum_{(i,j)\in\mc{E}_{\T}} \hat{I}(Y_i;Y_j)$, and the latter is equivalent with the output $\TCLn$ of Chow-Liu algorithm.\end{proof}

Although Corollary \ref{cor:PMLE} indicates optimality of the algorithm in some cases, it is well-known that under the presence of noise an MLE approach may be non-robust. To mitigate the effect of noise extra steps should be considered such as pre-processing, statistical learning of the noise by using pilot samples, or detecting and rejecting bad samples; see also~\citet[page 62]{zoubir2012robust}. When the CL algorithm is not consistent, we later discuss how appropriate processing procedures can provide a consistent variation of the algorithm for specific hidden models (Section \ref{sec:processing}).

\subsection{Analysis of the Error Event}
We continue by analyzing the event of incorrect reconstruction $\TCLn\neq \T$ (or $\TCL\neq \T$), which yields a sufficient condition for exact structure recovery.
\begin{proposition}\label{error_char}
The estimated tree $\TCLn\neq \T$ if and only if there exist two edges $e\equiv\left(w,\bar{w}\right)\in\T$ and $g\equiv (u,\bar{u})\in\TCLn$ such that  $e\notin\TCLn$, $g\notin\T$ and $e\in\tpath_{\T}\left(u,\bar{u}\right)$. Then also $g\in\tpath_{\TCLn}\left(w,\bar{w}\right)$.
\end{proposition}

Intuitively, exact recovery fails when there is at least one edge in the original tree $\T$ which does not appears in the estimated tree $\TCLn$. We refer the reader to the proof of Proposition \ref{error_char} by~\citet[Appendix F, ``Two trees lemma'', Lemma 9]{bresler2020learning}. For sake of space, we define the set $\fset$.

\begin{definition}[Feasibility set $\fset$] Let $e\equiv\left(w,\bar{w}\right)\in \mc{E}_{\T}$ be an edge and $u,\bar{u}\in\mc{V}_{\T}$ be a pair of nodes such that $e\in\tpath_{\T}\left(u,\bar{u}\right)$ and $|\tpath_{\T}\left(u,\bar{u}\right)|\geq 2$. The set of all such tuples $(e, u,\bar{u})$, is defined as \begin{align}
    \fset\triangleq \{&e,u,\bar{u}\in  \mc{E}_{\T}\times\mc{V}_{\T}\times\mc{V}_{\T} :e\in \tpath_{\T}(u,\bar{u}) \text{ and }|\tpath_{\T}\left(u,\bar{u}\right)|\geq 2\}.
\end{align}
\end{definition}

For the rest of the paper the pair of nodes $w$, $\bar{w}$ denotes the edge $e\equiv (w,\bar{w})\in\mc{E}_{\T}$. The error characterization of CL algorithm is expressed as follows: if $ \TCLn\neq \T$ then there exists $\lp(w,\bar{w}),u,\bar{u}\rp\in \fset$ such that\footnote{The event $\{\MI\lp Y_w;Y_{\bar{w}} \rp = \MI\lp Y_u;Y_{\bar{u}} \rp\}$ has non zero probability for certain cases, in this situation the MST arbitrarily chooses one of the edges $(w,\bar{w}),$ $(u,\bar{u})$. The choice of $(u,\bar{u})$ yields the error event $\TCLn\neq \T$.}
\begin{align*}
   \MI\lp Y_w;Y_{\bar{w}} \rp \leq \MI\lp Y_u;Y_{\bar{u}} \rp.
\end{align*} By negating the above statement, we get that if $\MI\lp Y_w;Y_{\bar{w}} \rp > \MI\lp Y_u;Y_{\bar{u}}\rp$ for all $\lp(w,\bar{w}),u,\bar{u}\rp\in \fset$ 
then $\TCLn= \T$.  The latter yields a sufficient condition for accurate structure estimation. 

\noindent\textbf{Sufficient condition for exact structure recovery:}
For exact structure recovery we need $\MI\lp Y_w;Y_{\bar{w}} \rp > \MI\lp Y_u;Y_{\bar{u}}\rp$ for all $\lp(w,\bar{w}),u,\bar{u}\rp\in \fset$, or equivalently
    \begin{align}\label{eq:exact_recovery_condition}
    I\lp Y_w;Y_{\bar{w}} \rp - I\lp Y_u;Y_{\bar{u}} \rp  
     > \MI\lp Y_u;Y_{\bar{u}} \rp-I\lp Y_u;Y_{\bar{u}} \rp -\MI\lp Y_w;Y_{\bar{w}} \rp +I\lp Y_w;Y_{\bar{w}} \rp.
    \end{align} 
Inequality \eqref{eq:exact_recovery_condition} allows us to derive a sufficient condition based on error estimates of the mutual information as follows.
\begin{proposition}
If \begin{align} \label{eq:MI_suff_cond2}
    \left|\MI\lp Y_\ell;Y_{\bar{\ell}} \rp -I\lp Y_\ell;Y_{\bar{\ell}} \rp\right|<\frac{1}{2}\min_{(e,u,\bar{u})\in\fset} \left\{I\lp Y_w;Y_{\bar{w}} \rp - I\lp Y_u;Y_{\bar{u}} \rp\right\},
    \end{align}  for all $\ell,\ell'\in \mc{V}$ then $\TCLn=\T$.
\end{proposition}

\noindent In fact \eqref{eq:MI_suff_cond2} implies \eqref{eq:exact_recovery_condition} and  \eqref{eq:exact_recovery_condition} implies $\TCLn=\T$. Inequality \eqref{eq:MI_suff_cond2} shows that if the error of mutual information estimates is less than a threshold statistic then exact structured recovery is guaranteed.

\subsection{Information Threshold and Properties}

We now define our quantity of interest for tree structured distributions, which we call the \emph{information threshold}. 
As well will see shortly, our sample complexity bounds for exact structure recovery via the CL algorithm depend on the distribution only through the information threshold, $\Ith_{\dagger}$. We first define $\Ith_{\dagger}$ and then show how it affects the difficulty of the structure estimation problem.

\begin{definition}[\textbf{Information Threshold (IT)}] \label{Def:Ith} 
Let $e\equiv\left(w,\bar{w}\right)\in \mc{E}_{\T}$ be an edge and $u,\bar{u}\in\mc{V}_{\T}$ be a pair of nodes such that $e\in\tpath_{\T}\left(u,\bar{u}\right)$. The \textit{information threshold} associated with the model $\np (\cdot)$ (see Section \ref{model_formulation}) is defined as
\begin{align}\label{eq:DI_definition1}
\Ith_{\dagger} \triangleq \frac{1}{2} \min_{(e,u,\bar{u})\in\fset} \lp I\lp Y_w;Y_{\bar{w}} \rp - I\lp Y_u;Y_{\bar{u}} \rp\rp.
\end{align}
\end{definition}\noindent The minimization in \eqref{eq:MI_suff_cond2} and \eqref{eq:DI_definition1} is with respect to the feasible set $\fset$ of the \textit{hidden} tree structure $\T$ of $\bX$. Note that the distribution of $\bnX$ does not factorize according to any tree~\citep{yeung2019information}, thus it is possible to have $\Ith_{\dagger}<0$.

When the data are noiseless, the gap between mutual informations that defines the information threshold will change. If the errors of the mutual information estimates of the noiseless variables satisfy the condition \begin{align} \label{eq:MI_suff_cond3}
    \left|\MI\lp X_\ell;X_{\bar{\ell}} \rp -I\lp X_\ell;X_{\bar{\ell}} \rp\right|<\frac{1}{2}\min_{(e,u,\bar{u})\in\fset} 
        \big( I\lp X_w X_{\bar{w}} \rp - I\lp X_u;X_{\bar{u}} \rp \big),
    \end{align} for all $\ell,\bar{\ell}\in\mc{V}$ then $\TCL=\T$, and \eqref{eq:MI_suff_cond3} is derived similarly to \eqref{eq:MI_suff_cond2}. The definition of the \textit{noiseless information threshold} naturally results from the previous condition.

\begin{definition}[\textbf{Noiseless IT}]\label{def:Ith_noisy}
The \textit{noiseless information threshold is defined as}
\begin{align}\label{eq:Ith_noisy}
\Ith \triangleq \frac{1}{2}\min_{(e,u,\bar{u})\in\fset} \big( I\lp X_w;X_{\bar{w}} \rp - I\lp X_u;X_{\bar{u}} \rp \big).
\end{align}
\end{definition}

\noindent Similarly to the noisy case, if 
\begin{align}\label{eq:MI_suff_cond}
    &\left|\MI\lp X_\ell;X_{\bar{\ell}} \rp -I\lp X_\ell;X_{\bar{\ell}} \rp\right|< \Ith \quad\forall \ell,\ell'\in\mc{V}, 
\end{align}
then $\T=\TCL$. The data processing inequality \citep{cover2012elements} shows that $\Ith\geq 0$. 
Also, Assumption \ref{ass:unique_tree} guarantees that $\Ith>0$.

\begin{proposition}[\textbf{Positivity}]\label{prop1}
If Assumption \ref{ass:unique_tree} holds then $\Ith>0$.
\end{proposition}
Since the values $ I (X_i,X_j)$ for $(i,j)\in\mc{E}$ are constant relative to $p$~\citep{tan2011learning}, $\Ith$ does not depend on $p$. The latter holds because of the locality property of $\Ith$.

\begin{proposition}[\textbf{Locality}]\label{prop_locality}
Assume that Assumption \ref{ass:unique_tree} holds. Let $(e^*,u^*,\bar{u}^*)\in\fset$ be a tuple such that \begin{align}
  (e^*,u^*,\bar{u}^*)=\argmin_{((w,\bar{w}),u,\bar{u})\in\fset}  I(X_{w};X_{\bar{w}})-I(X_{u};X_{\bar{u}}),\label{eq:argmin}
\end{align} then $u^*\equiv w^*$ or $u^*\equiv \tilde{w}^*$ and $\bar{u}\in\mc{N}_{\T}(w)$ or $\bar{u}\in\mc{N}_{\T}(\bar{w})$.
\end{proposition} 

\noindent We prove Propositions \ref{prop1} and \ref{prop_locality} in Section \ref{Appendix} of the Appendix.

\section{Recovering the Structure from Noisy Data}



We start by developing a finite sample complexity bound for exact structure learning with high probability, when noisy data are available. The structure learning condition in \eqref{eq:MI_suff_cond}, combined with results on concentration of mutual information estimators~\citep{antos2001convergence}, yields the following result.
\begin{theorem}[General Alphabets]\label{thm:sufficient_noisy}
Assume that $\bX\sim \p(\cdot)\in\mc{P}_{\mc{T}}$. Assume that noisy data $\bnX\sim \p_{\dagger} (\cdot)$ are generated by a randomized set of mappings $\mc{F}=\{F_i (X_i)=Y_i:i\in[p]\}$, and $\p_{\dagger} (\cdot)$ satisfies Assumption \ref{antos_ass} for some $c\geq 2,c_1>c_2>0$. Fix a number $\delta\in (0,1)$. If the number of samples of $\bnX$ satisfies the inequalities
\begin{align}\label{eq:sufficient_number_of_samples_q>2}
   \frac{n}{\log^2 n}\geq \frac{\max\{288\log \lp \frac{p}{\delta}\rp,4C^2\}}{\lp\Ith_{\dagger}\rp^2}
\end{align}
for a constant $C>0$, then Algorithm \ref{alg:Chow-Liu} with input $\mc{D}=\bnX^{1:n}$ returns $\TCLn=\T$ with probability at least $1-\delta$.
\end{theorem}

\noindent \textbf{Remarks}: First, the constant $C$ depends on the values of constants $c,c_1,c_2$ which are defined in Assumption \ref{antos_ass}. Specifically, $C=3c_2\big[ c_2 ^{(1-c)/c}+  c^{-1}\int^{\infty}_{c_1} u^{1/c-2}\log\lp eu/c_1\rp +1/c_1\big] $. 
The derivation of $C$ has been given by Antos and Kontoyiannis~\cite[Theorem 7]{antos2001convergence}. 
Additionally, note that $n/\log^2 n=\Omega\lp n^\eps \rp $ for any fixed $\eps\in (0,1)$. Therefore, the required number of samples $n$ with respect to $p$ and $\delta$ and for fixed $\Ith_{\dagger}$ scale as $\mc{O}\big(\log^{1+\zeta} (p/\delta)\big) $, for any choice of $\zeta>0$, whereas, for fixed $p$ and $\delta$, the complexity is of the order of $\Ith_{\dagger}$ is $\mc{O}\big( (\Ith_{\dagger})^{-2(1+\zeta)} \big)$, for any $\zeta>0$.
The proof of Theorem \ref{thm:sufficient_noisy} now follows.

\begin{proof}[Proof of Theorem \ref{thm:sufficient_noisy}]
To calculate the probability of the exact structure recovery we use a concentration inequality quantifying the rate of convergence of entropy estimators from Antos and Kontoyiannis~\citep{antos2001convergence}. In particular, they  show (\hspace{-0.2pt}\cite[Corollary 1]{antos2001convergence}) how the \textit{plug-in} entropy estimator $\hat{H}_n$ (say) is distributed around its mean $\mbb{E}[\hat{H}_n]$: For every $n\in\mbb{N}$ and $\eps>0$,
\begin{align}\label{eq:antos}
\P \left[ \left| \hat{H}_n - \E[ \hat{H}_n] \right|>\epsilon    \right]&\leq 2 e^{-n\epsilon^2/2\log^2 n}.
\end{align}
The plug-in entropy estimator is biased and, actually, $H\geq \mbb{E}[\hat{H}_n]$. Under their Assumption \ref{antos_ass}, in Theorem 7~\citep{antos2001convergence} they characterize the bias as follows. For $c \in [2,\infty)$ (which is the case of interest in this proof),
\begin{align}\label{eq:antos1}
H - \E[ \hat{H}_n]&=\mc{O}\big( n^{-1/2}\log n \big),
\end{align} 
and for $c \in (1,2)$,
    \begin{align}
   H - \E[ \hat{H}_n]&= \Theta \big( n^{\frac{1-c}{c}} \big).
    \end{align}

\noindent Then, for $\eps>Cn^{-1/2}\log n$, it is true that
    \begin{align}\nonumber
    &\P \left[ \left| \hat{H}_n - H] \right|>\epsilon    \right]\\&=\P \left[ \left| \hat{H}_n - \E[ \hat{H}_n]+ \E[ \hat{H}_n]- H] \right|>\epsilon    \right]\nonumber\\
    &\leq \P \left[ \left| \hat{H}_n - \E[ \hat{H}_n]\right|+\left| \E[ \hat{H}_n]- H] \right|>\epsilon\right]\nonumber\\
    &=  \P \left[ \left| \hat{H}_n - \E[ \hat{H}_n]\right|>\epsilon-\left| \E[ \hat{H}_n]- H] \right|\right]\nonumber\\
    &\leq \P \left[ \left| \hat{H}_n - \E[ \hat{H}_n]\right|>\epsilon-C\frac{\log n}{\sqrt{n}}\right]\nonumber\\
    &\leq 2 e^{-n\big(\epsilon-C\frac{\log n}{\sqrt{n}} \big)^2/2\log^2 n}\label{eq:entropy_estimate_bound},
\end{align} 
and the last inequality comes from \eqref{eq:antos} and \eqref{eq:antos1}. Notice that for non-trivial bounds we need the condition $\eps>Cn^{-1/2}\log n$. Further, $\eps$ is free parameter and we choose $\epsilon =\Ith_{\dagger}/3$, driven by property \eqref{eq:MI_suff_cond3}. This requires than $n$ has to be sufficiently large, such that the following inequality holds
    \begin{align}\label{eq:deltaI_inequality}
    \Ith_{\dagger}>3Cn^{-1/2}\log n.
    \end{align} 
Our goal is to find an upper on the probability of the event $\big\{\big|\hat{I}\lp Y_\ell;Y_{\bar{\ell}} \rp -I\lp Y_\ell;Y_{\bar{\ell}} \rp\big|>\Ith_{\dagger} \big\}$. By combining the above and applying the property $I(X;Y)=H(X)+H(Y)-H(X,Y)$ we have
\begin{align*}
&\P \Big[ \Big|\hat{I}\lp Y_\ell;Y_{\bar{\ell}} \rp -I\lp Y_\ell;Y_{\bar{\ell}} \rp\Big|>\Ith_{\dagger} \Big]\nonumber\\
&=\P \Big[ \Big| \hat{H}(Y_\ell)+\hat{H}(Y_{\bar{\ell}})-\hat{H}(Y_\ell,Y_{\bar{\ell}})-H(Y_\ell)-H(Y_{\bar{\ell}})+H(Y_\ell,Y_{\bar{\ell}}) \Big|>\Ith_{\dagger} \Big]\nonumber\\
&\leq \P \Big[ \Big| \hat{H}(Y_\ell)-H(Y_\ell)\Big| +\Big|\hat{H}(Y_{\bar{\ell}})-H(Y_{\bar{\ell}})\Big| +\Big|H(Y_\ell,Y_{\bar{\ell}})-\hat{H}(Y_\ell,Y_{\bar{\ell}})\Big|>\Ith_{\dagger} \Big]\nonumber\\
&\leq \P\! \Bigg[\! \Big\{ \Big| \hat{H}(Y_\ell)-H(Y_\ell)\Big|\!>\!\frac{\Ith_{\dagger}}{3}  \Big\}\!\bigcup\! \Big\{ \Big| \hat{H}(Y_{\bar{\ell}})-H(Y_{\bar{\ell}})\Big|\!>\!\frac{\Ith_{\dagger}}{3} \Big\}\!\bigcup\! \Big\{ \Big|H(Y_\ell,Y_{\bar{\ell}})-\hat{H}(Y_\ell,Y_{\bar{\ell}})\Big|\!>\!\frac{\Ith_{\dagger}}{3} \Big\}\!   \Bigg]\nonumber\\
&\leq 6e^{-n\big(\frac{\Ith_{\dagger}}{3}-C\frac{\log n}{\sqrt{n}}  \big)^2/2\log^2 n},\numberthis\label{eq:bound_on_prob_of_err}
\end{align*} 
where the last inequality is a consequence of \eqref{eq:entropy_estimate_bound}. To guarantee that the condition in \eqref{eq:MI_suff_cond3} holds, we apply the union bound on the events $\big\{ \big|\hat{I}\lp Y_\ell;Y_{\bar{\ell}} \rp -I\lp Y_\ell;Y_{\bar{\ell}} \rp\big|>\Ith_{\dagger}\big\}$, for all $\ell,\bar{\ell}\in\mc{V}$. Since there are $\binom{p}{2}$ pairs we define 
    \begin{align}
    \delta \triangleq  \binom{p}{2} 6 e^{-n\big(\frac{\Ith_{\dagger}}{3} 
    - C\frac{\log n}{\sqrt{n}}  \big)^2/2\log^2 n}.
    \label{eq:expr_delta}
\end{align} 
To conclude, for some fixed $\delta\in (0,1)$ if 
    \begin{align}\label{eq:proof_th1}
    \frac{n}{\log^2 n}\geq \frac{2\log \Big( \frac{6 \binom{p}{2} }{\delta} \Big)}{\Big(\frac{\Ith_{\dagger}}{3}-C\frac{\log n}{\sqrt{n}} \Big)^2} \quad\text{ and }\quad \Ith_{\dagger}>3C\frac{\log n}{\sqrt{n}},
\end{align} 
then the probability of exact recovery is at least $1-\delta$. The latter combined with the inequalities $8\log \lp p/\delta\rp>2\log \lp 6 \binom{p}{2}/\delta \rp$, $p\geq 3$ gives 
\begin{align}
   \frac{n}{\log^2 n}\geq \frac{72\log \lp \frac{p}{\delta}\rp}{\big(\Ith_{\dagger}-C\frac{\log n}{\sqrt{n}} \big)^2}\hspace{+0.1cm} \text{  and  }\hspace{+0.1cm} \Ith_{\dagger}>C\frac{\log n}{\sqrt{n}},
\end{align} thus it sufficient to have

\begin{align}
   &\frac{n}{\log^2 n}\geq \frac{72\log \lp \frac{p}{\delta}\rp}{\big(\Ith_{\dagger}-C\frac{\log n}{\sqrt{n}} \big)^2}\hspace{+0.1cm} \text{  and  }\hspace{+0.1cm} \frac{\Ith_{\dagger}}{2}\geq C\frac{\log n}{\sqrt{n}}, \implies\\
   &\frac{n}{\log^2 n}\geq \frac{288\log \lp \frac{p}{\delta}\rp}{\big(\Ith_{\dagger}\big)^2}\hspace{+0.1cm} \text{  and  }\hspace{+0.1cm} \frac{n}{\log^2 n}\geq \lp \frac{2C}{\Ith_{\dagger}}\rp^2.
\end{align} The last statement gives the statement of the theorem.
\end{proof}




Theorem \ref{thm:sufficient_noisy} characterizes the sample complexity for models with \textit{either} countable \textit{or} finite alphabets. By restricting our setting to finite alphabets Assumption \ref{antos_ass} is not required and we have the following result. The proof is virtually identical to that of Theorem \ref{thm:sufficient_noisy}, and is omitted.

\begin{theorem}[Finite Alphabets]\label{Rmrk:finite_alphabet_suff}
Assume that the random variable $\bnX$ (as defined in Theorem \ref{thm:sufficient_noisy}) has finite support. Fix a number $\delta\in (0,1)$. There exists a constant $C>0$, independent of $\delta$ such that, if the number of samples of $\bnX$ satisfies the inequalities
\begin{align}\label{eq:sufficient_number_of_samples_finite}
   \frac{n}{\log^2 n}\geq \frac{288\log \lp \frac{p}{\delta}\rp}{\big(\Ith_{\dagger}\big)^2}\hspace{+0.1cm} \text{  and  }\hspace{+0.1cm} n\geq \lp \frac{2C}{\Ith_{\dagger}}\rp^2,
\end{align}
then Algorithm \ref{alg:Chow-Liu} with input $\mc{D}=\bnX^{1:n}$ returns $\TCLn=\T$ with probability at least $1-\delta$.
\end{theorem}

\noindent Lastly, the corresponding variation of Theorem \ref{thm:sufficient_noisy} when Assumption \ref{antos_ass} holds for $c\in (1,2)$ (in the general case of countable alphabets) follows.
\begin{theorem}[Countable Alphabets, $c<2$]\label{THM_c>2}
Assume that $\bX\sim \p(\cdot)\in\mc{P}_{\mc{T}}$. Assume that noisy data $\bnX\sim \p_{\dagger} (\cdot)$ are generated by a randomized set of mappings $\mc{F}=\{F_i (X_i)=Y_i:i\in[p]\}$, and $\p_{\dagger} (\cdot)$ satisfies the Assumption \ref{antos_ass} for some $c\in (1,2),c_1>c_2>0$. Fix $\delta\in (0,1)$. There exists a constant $C>0$ independent of $\delta$ such that, if  $\Ith_{\dagger}>0$ and the number of samples $n$ of $\bnX$ satisfies the inequalities
\begin{align}\label{eq:sufficient_number_of_samples_noise}
   \frac{n}{\log^2 n}\geq \frac{288\log \lp \frac{p}{\delta}\rp}{\big(\Ith_{\dagger}\big)^2}\hspace{+0.1cm} \text{  and  }\hspace{+0.1cm} n\geq \lp \frac{2C}{\Ith_{\dagger}}\rp^\frac{c}{c-1}
\end{align}
 then Algorithm \ref{alg:Chow-Liu} with input the noisy data $\mc{D}=\bnX^{1:n}$ returns $\TCLn=\T$ with probability at least $1-\delta$. 
\end{theorem}


As a byproduct of our analysis we also derive the sample complexity bound for the noiseless case as well. Now the bound involves the (noiseless) information threshold $\Ith$.
\begin{theorem}\label{Thm:1}
Assume that $\bX\sim \p(\cdot)\in\treedistr$ and $\p (\cdot)$ satisfies the Assumption \ref{antos_ass} for some $c\geq 2,c_1>c_2>0$. Fix a number $\delta\in (0,1)$. There exists a constant $C>0$, independent of $\delta$ such that, if the number of samples of $\bX$ satisfies the inequality
\begin{align}\label{eq:sufficient_number_of_samples}
   \frac{n}{\log^2 n}\geq \frac{\max\{288\log \lp \frac{p}{\delta}\rp,4C^2\}}{\lp\Ith\rp^2}
\end{align}then Algorithm \ref{alg:Chow-Liu} with input $\mc{D}=\bX^{1:n}$ returns $\TCL=\T$ with probability at least $1-\delta$. The relationship between $C$ and $c,c_1,c_2$ is given similarly to that of Theorem \ref{thm:sufficient_noisy}. 
\end{theorem} 


This result allows us to compare the sample complexity of the noiseless and noisy setting. Let $n$ and $\nn$ denote the sufficient number of samples of $\bX$ and $\bnX$ respectively and consider $p,\delta$ fixed, then the ratio $n/\nn$ is $\mc{O}\big((\Ith/\Ith_\dagger)^{-2(1+\zeta)}\big) $ for all $\zeta>0$. The latter shows how $\nn$ changes relative to $n$ under the same probability of success for both settings (noiseless and noisy). 
The proof of Theorem \ref{Thm:1} is similar to the proof of Theorem \ref{thm:sufficient_noisy}.
\begin{proof}[Proof of Theorem \ref{Thm:1}] The difference is introduced by the event in \eqref{eq:MI_suff_cond}. That is, the error on the mutual information estimates should be less than the noiseless information threshold $\Ith$.
Note that for the entropy estimates of $\bX$, equations \eqref{eq:antos} up to \eqref{eq:entropy_estimate_bound} hold with possibly different constants $c',C'$ than those of the observable layer (see Assumption \ref{antos_ass}). Here we consider the case where $c'\geq 2$ (the case $c'\in (1,2)$ is similar, see also the proof of Theorem \ref{thm:sufficient_noisy}),
and the corresponding bound on the estimation error $\eps$ has to be at most $\Ith/3$. Thus, \eqref{eq:deltaI_inequality} becomes \begin{align}
\Ith>3C'n^{-1/2}\log n,
\end{align} and \eqref{eq:bound_on_prob_of_err} now is written as \begin{align}
&
\hspace{-15pt}\P \left[ \left|\hat{I}\lp X_\ell;X_{\bar{\ell}} \rp -I\lp X_\ell;X_{\bar{\ell}} \rp\right|>\Ith \right]\leq 6e^{-n\Big(\frac{\Ith}{3}-C'\frac{\log n}{\sqrt{n}} \Big)^2/2\log^2 n}.
\end{align} Finally, by applying union bound over the pairs $\ell,\bar{\ell}\in\mc{V}$ we derive the statement of Theorem \ref{thm:sufficient_noisy}, by following the equivalent steps of \eqref{eq:expr_delta} and \eqref{eq:proof_th1}. The latter completes the proof. \end{proof}

\section{Converse: Information Threshold as a Fundamental Quantity}
In this section we provide the statement of the converse of our main result Theorem \ref{thm:sufficient_noisy}. We define the class of hidden tree-structured models with bounded (from below) information threshold of the hidden layer by an absolute positive constant and bounded absolute information threshold for the observable layer by a fixed $\Delta>0$ as\footnote{The information threshold ${\bf I}_{\dagger,\mathrm{M}}^{o} $ of the observable layer can be either positive or negative for each model $M$ in the class ${\cal C}^{\cal T}$.} \begin{equation}
{\cal C}^{{\cal T}}_{\Delta}\triangleq\{\mathrm{M}: |{\bf I}_{\dagger,\mathrm{M}}^{o}| \geq \Delta>0 \}.
\end{equation} The next result provides a lower bound for the necessary number of samples for tree-structure learning from noisy observations.

\begin{theorem}\label{thm:converse_noisy}
There exist absolute constants $C,\eps_0>0$ such that for any $\bold{I}_{\min}\in (0,\eps_0)$ and for any estimator $\Phi:\bnX^{1:n}\to \mc{T}$, if $n< C/\bold{I}_{\min}^2$, then the worst-case probability of incorrect structure recovery over all hidden tree-structured models in $ {\cal C}^{{\cal T}}_{\bold{I}_{\min}}$ is at least $1/2$. In other words, it is true that if $n< C/\bold{I}_{\min}^2$, then
\begin{equation}
\inf_{\Phi:\bnX^{1:n}\to \mc{T}}\sup_{\mathrm{M}\in{\cal C}_{\bold{I}_{\min}}^{{\cal T}}}{\mbb{P}}\big(\Phi(\bnX_{\mathrm{T}_{\mathrm{M}}}^{1:n})\neq\mathrm{T}_{\mathrm{M}}\big)\ge\dfrac{1}{2}.
\end{equation}
Additionally, the supremum is attained.
\end{theorem}

\begin{figure}%
    \centering
    \subfloat[\centering Model $M_0$]{{\includegraphics[width=6cm]{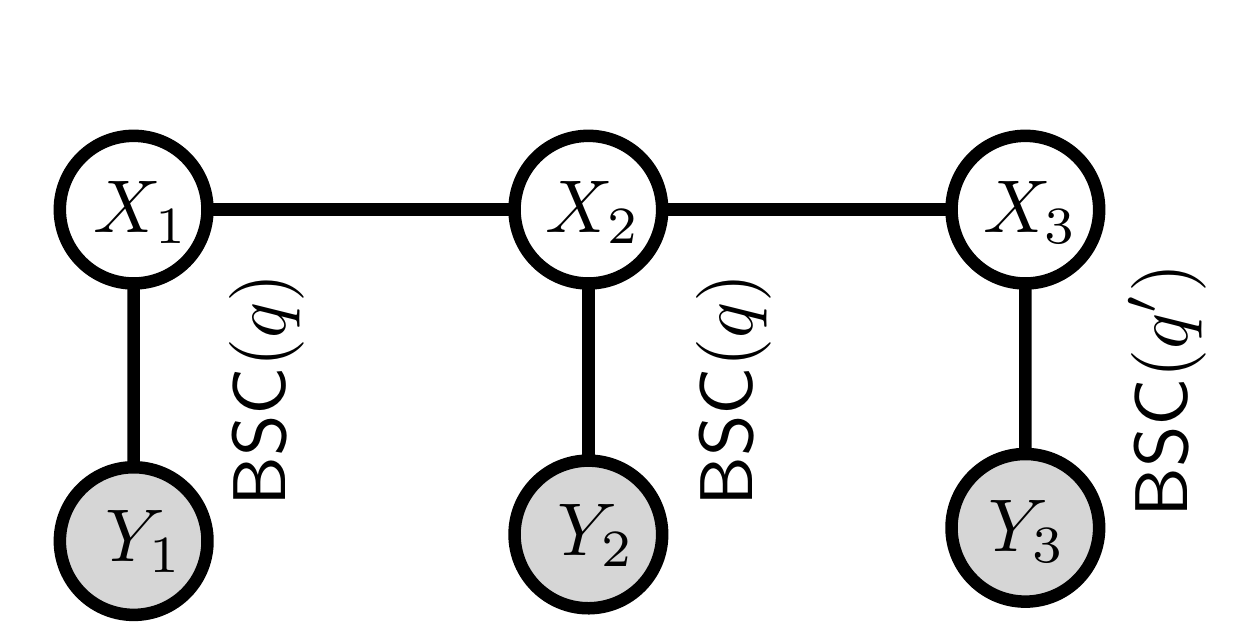} }}%
    \qquad
    \subfloat[\centering  Model $M_1$]{{\includegraphics[width=6cm]{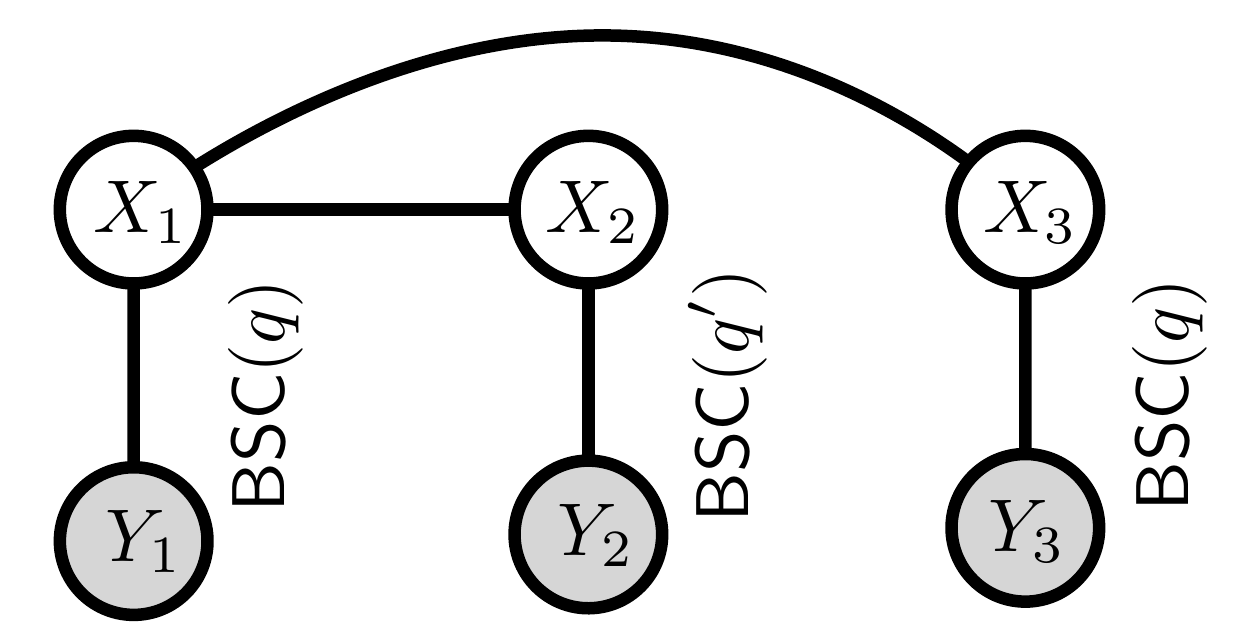} }}%
     \qquad
    \subfloat[\centering  Model $M_2$]{{\includegraphics[width=6cm]{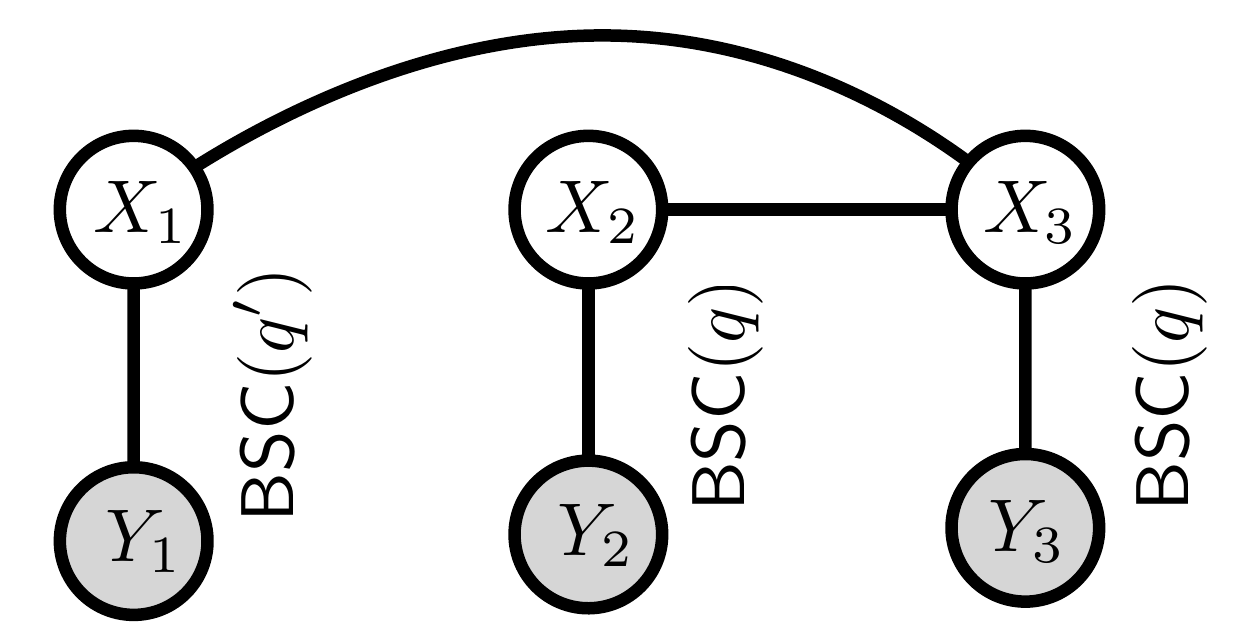} }}
    \caption{Hidden tree-structured models $M_0$, $M_1$, $M_2$}%
    \label{fig:noisy_converse}%
\end{figure}

\begin{proof}[Proof of Theorem \ref{thm:converse_noisy}]
We define the set of hidden tree-structured models $M_0$, $M_1$ and $M_2$ (see Figure \ref{fig:noisy_converse}) as follows: The hidden layer of $M_0$ is $X_1 - X_2 - X_3$, and $\mc{E}_{M_0}=\{ (1,2), (2,3) \}$, the hidden layer of $M_1$ is $X_2-X_1-X_3$ and $\mc{E}_{M_1}=\{ (1,2), (1,3) \}$, the hidden layer of $M_2$ is $X_2-X_3-X_1$ and $\mc{E}_{M_2}=\{ (2,3), (1,3) \}$. Further for all models, $X_i\in \{-1,+1\}$ for $i=1,2,3$ and $\E[X_i X_j]\triangleq c\in (0,1)$ for any $(i,j)\in\mc{E}_{M_k}$ and $k\in\{0,1,2\}$. Note that the information threshold $\Ith_{M_0}$ of the hidden layer is strictly positive for any $c\in (0,1)$, since 
\begin{align}
   \!\!\!\!\! \Ith_{M_0}\! &=\! I_{M_0}(X_1; X_2) - I_{M_0}(X_1; X_3)\nonumber \\&=\frac{1}{2} \log \lp \frac{\lp 1-\E\left[X_{1}X_{2}\right]\rp^{1-\E\left[X_{1}X_{2}\right]} \lp 1+\E\left[X_{1}X_{2}\right]\rp^{1+\E\left[X_{1}X_{2}\right]}}{\lp 1-\E\left[X_{1}X_{3}\right]\rp^{1-\E\left[X_{1}X_{3}\right]} \lp 1+\E\left[X_{1}X_{3}\right]\rp^{1+\E\left[X_{1}X_{3}\right]}}\rp, \label{eq:ITH_Hidden}
\end{align} and by construction $\Ith_{M_1} = I_{M_1}(X_1; X_2) - I_{M_1}(X_2; X_3)\equiv\Ith_{M_0}$, $\Ith_{M_2} = I_{M_2}(X_2; X_3) - I_{M_2}(X_1; X_2)\equiv\Ith_{M_0}$. The observable data are generated by binary symmetric channels with cross-over probabilities $1-q$, $1-q'$ and $q=0.75+\eps$ and $q'=0.75-\eps$. The binary noise variables $N_i\in \{-1,+1\}$ are multiplicative, independent from each other and independent from the hidden variables and generate the corresponding observable as $Y_i=N_i \times X_i$ (see Figure \ref{fig:noisy_converse}). Specifically, for the model $M_0$ we have $\P (N_1=+1)=\P (N_2=+1)=q$ and $\P (N_3=+1)=q'$, for the model $M_1$ $\P (N'_1=+1)=\P (N'_3=+1)=q$ and $\P (N'_2=+1)=q'$, and for the model $M_2$ $\P (N_2=+1)=\P (N_3=+1)=q$ and $\P (N_1=+1)=q'$.\footnote{The construction of the hidden models $M_0,M_1,M_2$ is similar to~\citet[Lemma 7.1]{bhattacharyya2020near}. The crucial difference is that our construction involves hidden layers with $\Ith>0$, while in their Lemma 7.1 the information threshold of the hidden is zero and the construction is inappropriate for the structure learning problem.} 

The distribution of any pair $Y_i,Y_j\in\{-1,+1\}$ is \begin{align}
    \p (y_i,y_j)&=\frac{1+\E[Y_i Y_j]y_i y_j}{4}, \quad y_i,y_j\in\{-1,+1\}.
\end{align} We find the second order moments $\E [Y_i Y_j]$ of the model $M_0$ as follows:
    \begin{align}
    \E [N_1]&= \E [N_2]=+1\times\lp\frac{3}{4}+\eps\rp -1\times\lp\frac{1}{4}-\eps \rp= \frac{1}{2}+2\eps,\\
    \E [N_3] &= +1\times\lp\frac{3}{4}-\eps\rp -1\times\lp\frac{1}{4}+\eps \rp= \frac{1}{2}-2\eps.\\
    \E[Y_1 Y_2]&=\E[N_1 X_1 N_2 X_2]=\E [N_1] \E [N_2] \E [X_1 X_2]= c \lp \frac{1}{2}+2\eps\rp^2,\\
    \E[Y_2 Y_3]&= \E[N_2 X_2 N_3 X_3]=c \lp\frac{1}{4} -4\eps^2 \rp,\\
    \E[Y_1 Y_3]&=\E[N_1 X_1 N_3 X_3]= c^2 \lp\frac{1}{4} -4\eps^2 \rp.
\end{align} We combine the above and we choose $c=1-\eps$ to evaluate the information threshold $\Ith_{\dagger,M_0}$\begin{align*}
    &\Ith_{\dagger,M_0}=I(Y_1;Y_2)-I(Y_1;Y_3)= H(Y_1;Y_3)-H(Y_1;Y_2)\\
    &= - \lp \frac{1+c^2 \lp \frac{1}{4}-4\eps^2 \rp}{2} \rp\log\lp \frac{1+c^2 \lp \frac{1}{4}-4\eps^2 \rp}{4} \rp\\& \quad- \lp \frac{1-c^2 \lp \frac{1}{4}-4\eps^2 \rp}{2} \rp\log\lp \frac{1-c^2 \lp \frac{1}{4}-4\eps^2 \rp}{4} \rp\\
    & \quad  + \lp \frac{1+c \lp\frac{1}{2} +2\eps \rp^2}{2} \rp\log\lp \frac{1+c \lp\frac{1}{2} +2\eps \rp^2}{4} \rp+ \lp \frac{1-c \lp\frac{1}{2} +2\eps \rp^2}{2} \rp\log\lp \frac{1-c \lp\frac{1}{2} +2\eps \rp^2}{4} \rp \\
    &= \frac{9}{8}\log\lp \frac{5}{3} \rp\eps +\mc{O}(\eps^2).\numberthis \label{eq:inf_thres_M1}
\end{align*} Similarly for the model $M_1$,
 \begin{align}
    &\E [N_1]= \E [N_3]= +1\times\lp\frac{3}{4}+\eps\rp -1\times\lp\frac{1}{4}-\eps \rp= \frac{1}{2}+2\eps,\\
    &\E [N_2] = +1\times\lp\frac{3}{4}-\eps\rp -1\times\lp\frac{1}{4}+\eps \rp= \frac{1}{2}-2\eps,\\
    &\E[Y_1 Y_2]=\E[N_1 X_1 N_2 X_2]=\E [N_1] \E [N_2] \E [X_1 X_2]= c \lp\frac{1}{4} -4\eps^2 \rp,\\
    &\E[Y_2 Y_3]= \E[N_2 X_2 N_3 X_3]=c^2 \lp\frac{1}{4} -4\eps^2 \rp,\\
    &\E[Y_1 Y_3]=\E[N_1 X_1 N_3 X_3]= c \lp \frac{1}{2}+2\eps\rp^2.
\end{align} The information threshold $\Ith_{\dagger,M_1}$ is
\begin{align*}
    &\Ith_{\dagger,M_1}=I(Y_1;Y_2)-I(Y_2;Y_3)= H(Y_2;Y_3)-H(Y_1;Y_2)\\
    &= -2 \lp \frac{1+c^2 \lp\frac{1}{4} -4\eps^2 \rp}{4} \rp\log\lp \frac{1+c^2 \lp\frac{1}{4} -4\eps^2 \rp}{4} \rp  \\&\quad-2 \lp \frac{1-c^2 \lp\frac{1}{4} -4\eps^2 \rp}{4} \rp\log\lp \frac{1-c^2 \lp\frac{1}{4} -4\eps^2 \rp}{4} \rp\\
    & \quad  +2 \lp \frac{1+c \lp\frac{1}{4} -4\eps^2 \rp}{4} \rp\log\lp \frac{1+c \lp\frac{1}{4} -4\eps^2 \rp}{4} \rp \\&\quad +2 \lp \frac{1-c \lp\frac{1}{4} -4\eps^2 \rp}{4} \rp\log\lp \frac{1-c \lp\frac{1}{4} -4\eps^2 \rp}{4} \rp \\
    &= \frac{\log\frac{5}{3}}{8}\eps+\mc{O}(\eps^2). \numberthis\label{eq:inf_thres_M2}
\end{align*} As a consequence, \eqref{eq:inf_thres_M1} and \eqref{eq:inf_thres_M2} give $\Ith_{\dagger,M_0} = \Theta(\eps)$ and $\Ith_{\dagger,M_1} = \Theta(\eps)$.  In contrast, \eqref{eq:ITH_Hidden} gives $\Ith=(\eps-1)^2/2+\mc{O}((\eps-1)^4)$. 
Further, the joint distributions of the models $M_0, M_1$ are given by \begin{align*}
    \p_{M_0}(y_1,y_2,y_3)&=\frac{1}{8}\left[1+\E[Y_1 Y_2]y_1 y_2 +\E[Y_2 Y_3]y_2 y_3+ \E[Y_1 Y_3]y_1 y_3\right]\\
    &=\frac{1}{8}\left[1+c \lp \frac{1}{2}+2\eps\rp^2 y_1 y_2 +c \lp\frac{1}{4} -4\eps^2 \rp y_2 y_3+ c^2 \lp\frac{1}{4} -4\eps^2 \rp y_1 y_3\right],\numberthis\label{eq:joint_M0}
\end{align*} \begin{align*}
    \p_{M_1}(y_1,y_2,y_3)&=\frac{1}{8}\left[1+\E[Y_1 Y_2]y_1 y_2 +\E[Y_2 Y_3]y_2 y_3+ \E[Y_1 Y_3]y_1 y_3\right]\\
    &=\frac{1}{8}\left[1+c \lp\frac{1}{4} -4\eps^2 \rp y_1 y_2 +c^2 \lp\frac{1}{4} -4\eps^2 \rp y_2 y_3+ c \lp \frac{1}{2}+2\eps\rp^2 y_1 y_3\right],\numberthis\label{eq:joint_M1}
\end{align*} and for both \eqref{eq:joint_M0}, \eqref{eq:joint_M0} $ y_1,y_2,y_3\in\{-1,+1\}$. We use \eqref{eq:joint_M0} and \eqref{eq:joint_M1} to evaluate the probabilities\begin{align*}
    &\p_{M_0}(+1,+1,+1)=\p_{M_0}(-1,-1,-1)= \frac{7}{32} + \frac{1}{8}\eps+\mc{O}(\eps^2)
    \\
    &\p_{M_0}(-1,+1,-1)=\p_{M_0}(+1,-1,+1)= \frac{3}{32} - \frac{1}{4}\eps +\mc{O}(\eps^2)
    \\
    &\p_{M_0}(-1,+1,+1)=\p_{M_0}(+1,-1,-1)=\frac{3}{32} - \frac{3}{16}\eps +\mc{O}(\eps^2)
    \\
    &\p_{M_0}(-1,-1,+1)=\p_{M_0}(+1,+1,-1)= \frac{3}{32} + \frac{5}{16}\eps +\mc{O}(\eps^2)
    \\
    &\p_{M_1}(+1,+1,+1)=\p_{M_1}(-1,-1,-1)=\frac{7}{32}+\frac{1}{8}\eps+\mc{O}(\eps^2)
    \\
    &\p_{M_1}(-1,+1,-1)=\p_{M_1}(+1,-1,+1)= \frac{3}{32} + \frac{5}{16}\eps+\mc{O}(\eps^2)
    \\
    &\p_{M_1}(-1,+1,+1)=\p_{M_1}(+1,-1,-1)= \frac{3}{32} -\frac{1}{4}\eps+\mc{O}(\eps^2)
    \\
    &\p_{M_1}(-1,-1,+1)=\p_{M_1}(+1,+1,-1)= \frac{3}{32} -\frac{3}{16}\eps+\mc{O}(\eps^2)
\end{align*} By the definition of the KL divergence \begin{align}
    \KL( p_{M_1} || p_{M_0} )=\frac{73}{24}\eps^2 + \frac{1129}{72}\eps^3 + \mc{O}(\eps^4).\label{eq:KL_order}
\end{align} Similarly for the model $M_2$ we have \begin{align}
    \Ith_{\dagger, M_2}= I_{M_2}(Y_2;Y_3)-I_{M_2}(Y_1;Y_2)=\frac{9\log(\frac{5}{3})}{8}\eps+\mc{O}(\eps^2)=\Theta (\eps),\label{eq:IthM3}
\end{align} \begin{align*}
    \p_{M_2}(y_1,y_2,y_3)=\frac{1}{8}\left[1+c \lp \frac{1}{2}+2\eps\rp^2 y_3 y_2 +c^2 \lp\frac{1}{4} -4\eps^2 \rp y_2 y_1+ c \lp\frac{1}{4} -4\eps^2 \rp y_1 y_3\right],\numberthis\label{eq:joint_M2}
\end{align*} and we find \begin{align}
    \KL( p_{M_2} || p_{M_0} )=\frac{73}{24}\eps^2 + \frac{1057}{72}\eps^3 + \mc{O}(\eps^4).\label{eq:KL_order_2}
\end{align}

Next, we use Fano's inequality (see~\citet[Corollary 2.6]{tsybakov2009introduction}).  Fix $ L = 2$ and let $\P_{M_0}$, $\P_{M_1},  \P_{M_2}$ denote the probability laws of $\bnX$ under models $M_0, M_1$ and $M_2$ respectively, and consider $n$ i.i.d. observations $\bnX^{1:n}$.  If \begin{align}
    n<  \frac{\alpha\log L}{\frac{1}{L+1}\sum^{L}_{j=1}  \KL(\P_{M_j}||\P_{M_0})}, 
\end{align} with $\alpha\in (0,1)$ then \begin{align}
    \inf_{\Phi}\max_{0\leq j\leq L} \P_{M_j} \left[ \Phi (\bnX^{1:n}) \neq \T_{M_j} \right]\geq \frac{\log (L+1)-\log(2)}{\log(L)}-\alpha,
\end{align}
where the infimum is relative to all estimators (statistical tests) $\Phi:\mc{Y}^{ n}\to \{0,1,\ldots,L\}$. Specifically, for $L=2$ and \begin{align}
    \tilde{\alpha} = \frac{\log (L+1)-\log(2)}{\log(L)}-\frac{1}{2}\in (0,1),
\end{align} if \begin{align}
    n<  \frac{\tilde{\alpha}\log 2}{\frac{2}{3} \max \{ \KL(\P_{M_1}||\P_{M_0}),\KL(\P_{M_2}||\P_{M_0}) \}  }, 
\end{align} then \begin{align}
 \inf_{\Phi}\sup_{M\in{\cal C}_{\Delta_\eps }^{{\cal T}}}{\mbb{P}}\big(\Phi(\bnX^{1:n})\neq\mathrm{T}_{M}\big)  \geq \inf_{\Phi}\max_{0\leq j\leq L} \P_{M_j} \left[ \Phi (\bnX^{1:n}) \neq \mathrm{T}_{M_j} \right]\geq \frac{1}{2}.
\end{align} Finally, $ \KL( p_{M_1} || p_{M_0} )=\mc{O}(\eps^2)$, $ \KL( p_{M_2} || p_{M_0} )=\mc{O}(\eps^2)$ from \eqref{eq:KL_order}, \eqref{eq:KL_order_2} and $\Ith_{\dagger,M_0} = \Theta(\eps)$, $\Ith_{\dagger,M_1} = \Theta(\eps)$, $\Ith_{\dagger,M_2} = \Theta(\eps)$ from \eqref{eq:inf_thres_M1}, \eqref{eq:inf_thres_M2}, \eqref{eq:IthM3}. Thus there exists $\eps_0\in (0,1)$ and $\tilde{C}>0$ such that for any $\eps\in (0,\eps_0)$ it is true that $\min\{ {\Ith_{\dagger,M_0},\Ith_{\dagger,M_1},\Ith_{\dagger,M_2}\}\geq \tilde{C}\eps}\triangleq \Delta (\eps)$. Then the statement of the theorem follows.\end{proof}

\section{Inconsistency of CL Algorithm and the Effect of Pre-Processing}\label{sec:processing}


The information threshold $\Ith_{\dagger}$ in \eqref{eq:Ith_noisy} appears in the condition for exact structure recovery using the CL algorithm \eqref{eq:MI_suff_cond3}: Algorithm \ref{alg:Chow-Liu} is consistent as long as $\Ith_{\dagger}>0$. To be more precise, for every $\Ith_{\dagger}>0$ there exists $N\in\mbb{N}$ such that if $n>N$ then \eqref{eq:MI_suff_cond3} holds with high probability. In fact, in Theorem \ref{thm:sufficient_noisy} the condition $\Ith_{\dagger}>0$ is necessary. Under the assumption of $\Ith>0$ (see Assumption \ref{ass:unique_tree}), it is not guaranteed that $\Ith_{\dagger}>0$. If $\Ith_{\dagger}<0$ then the CL is not consistent in the sense that
    \begin{align}\label{eq:inconsistency}
     \P\lp\lim_{n\to \infty}\mc{E}_{\TCL_{\dagger}(\mc{D}=\mbf{Y}^{1:n})} \neq \mc{E}_{\T}\rp=1.
\end{align}  On the other hand if $\Ith_{\dagger}=0$, then ties are broken arbitrarily and the probability of missing an edge does not decrease as $n$ increases. In what follows, we provide sufficient conditions which ensure that $\Ith_{\dagger}>0$. In particular, whenever $\Ith_{\dagger}<0$, we show that for certain hidden models, an appropriate processing can be introduced as an extra step to overcome the inconsistency of the CL algorithm. As a result, the output of CL algorithm under the appropriate processing will converge to the original tree $\T$ of the hidden layer.

First observe that if for all pairs of nodes  $\big( (k,l),(m,r)\big) \in\mc{V}^2\times \mc{V}^2$ the inequalities $ I\lp X_k;X_\ell \rp< I\lp X_m;X_r \rp$ and $ I\lp Y_k;Y_\ell \rp< I\lp Y_m;Y_r \rp$ simultaneously hold, then $\Ith_{\dagger}>0$. The last statement yields the required processing for the cases for which the algorithm is not consistent. Specifically, we would like to find weights $W_{i,j}$, that we consider as new input of the maximum spanning tree in Algorithm \ref{alg:Chow-Liu}. \begin{definition}\label{ordering_processed}
 A processing procedure with input $n$ noisy data, is called appropriate if it returns a set of edge weights $ \{ W^n_{k,\ell}\}_{(k,l)\in\mc{V}^2}$, and there exists $N\in\mbb{N}$ such that for every $\tilde{n}>N$ and every tuple $\big( (k,l),(m,r)\big) \in\mc{V}^2\times \mc{V}^2$
 \begin{align}
    & I\lp X_k;X_\ell \rp< I\lp X_m;X_r \rp \iff W^{\tilde{n}}_{k,\ell}< W^{\tilde{n}}_{m,r}.
\end{align}
\end{definition}

\noindent Such processing procedures are inherently model-based, and the required processing is tied to the underlying hidden model. In the next section, we show that in different models of interest, the noise can affect the original edge weights $I(X_i,X_j)$ in dissimilar ways; in fact, whether there exist universal rules for processing procedures for large classes of hidden models is currently unknown, and is an interesting problem for future work. 

\subsection{Relaxing the Condition $\Ith_{\dagger}>0$ for Two Non-Trivial Channels
}\label{enforcing}

\begin{figure}
\centering
\includegraphics[height=1.65in]{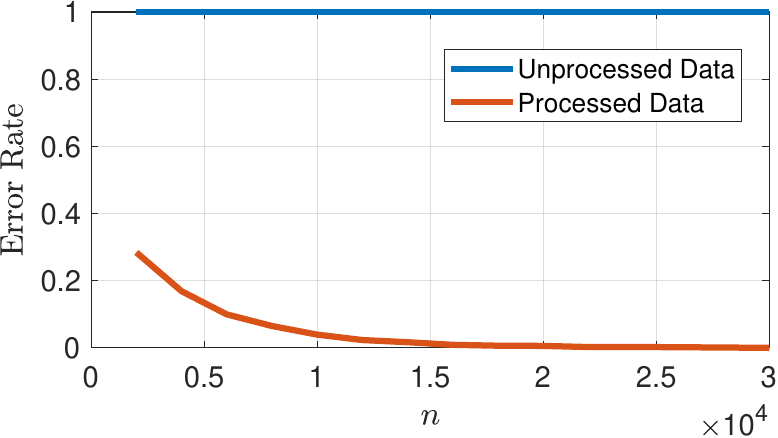}%
\hspace{+0.1cm}
\includegraphics[height=1.65in]{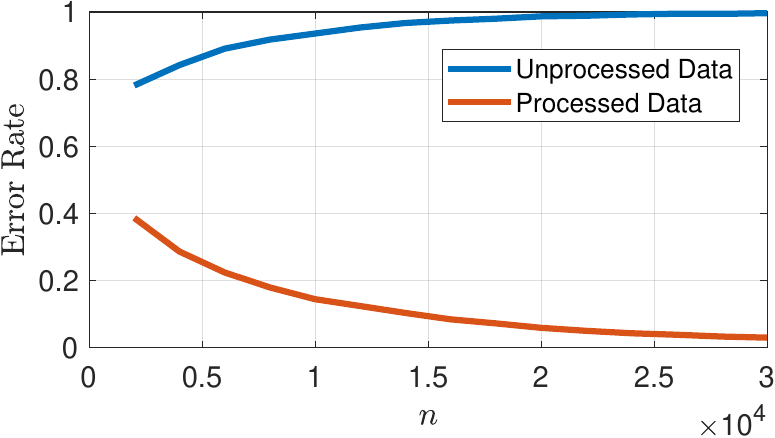}%
\caption{Estimate of the probability of incorrect recovery for a hidden Markov model with $3$ nodes, and $\Ith_{\dagger}<0$ (originally), for different values of $n\in [10^3 \times 10^4]$, before and after processing. Right: $q_1=0.01$, $q_2=q_3=0.3$, Left:  $q_1=0.2$, $q_2=q_3=0.25$.}
\label{fig:negative_Ith, simulations}
\end{figure}


To illustrate the case where $\Ith_{\dagger}<0$, we present and discuss scenarios of interest that involve hidden tree-structured models for which Algorithm \ref{alg:Chow-Liu} succeeds with high probability, only if an appropriate processing is being applied. Specifically, we provide rules for identifying the feasibility of structure recovery. In contrast with the condition $\Ith_{\dagger}>0$, these rules do not require knowledge of the structure or the unknown value $\Ith_{\dagger}$ beforehand. As a special case, the models in Sections \ref{BSC_example} and \ref{M-ary_arasure} with identically distributed noise always satisfy the conditions, and no processing is required. On the other hand, for non-identically distributed noise, structure learning from raw data may be infeasible. We show how to distinguish such models and we provide the appropriate processing under the cases of interest. 

\subsubsection{Tree-Structured Binary Data with Non-Identically Distributed Noise}\label{BSC_example}
Consider the hidden tree structure with hidden nodes $X_i$ and observable nodes $Y_i$ such that  $X_i,N_i\in\{-1,+1\}$ and $Y_i=N_i X_i$, for all $i\in\mc{V}$. The distribution of the noise is not identically distributed as \begin{align}
    \P (N_i=-1)=1-\P(N_i=+1)=q_i\in (0,1/2).
\end{align} If $I(Y_w;Y_{\bar{w}})-I(Y_u;Y_{\bar{u}})< 0$ for any tuple $(w,w),u,\bar{u}\in\fset$ then the CL is not consistent (see \eqref{eq:inconsistency}), and \eqref{eq;MI_BSC} gives \begin{align}\label{eq:ex3_cond}
    \frac{(1-q_w)(1-q_{\bar{w}})}{(1-q_u)(1-q_{\bar{u}})}
    <\frac{\E[ X_u;X_{\bar{u}}]}{\E[X_w X_{\bar{w}}]}.
\end{align} 

\noindent The inequality above involves the set $\fset$, that is associated with the unknown tree structure. Thus we cannot check if \eqref{eq:ex3_cond} holds before running the CL algorithm. However, we show that we can relax the above condition as  \begin{align}\label{eq:BSC_suff_cond}
   \!\!\!\! \frac{(1-2q_{i})}{(1-2q_{j})}\in  \lp \sqrt{\max_{(i,j)\in\mc{V}} |\mbb{E}[X_i X_j]|}, \frac{1}{\sqrt{\underset{(i,j)\in\mc{V}}{\max} |\mbb{E}[X_i X_j]|}}\rp \quad \forall i,j\in\mc{V},
\end{align} and then $\T\to\TCL_{\dagger}$. The proof of \eqref{eq:BSC_suff_cond} is given in Appendix \ref{BSC_non_id_noise}. Condition \eqref{eq:BSC_suff_cond} provides a testing rule for the feasibility of tree-structure estimation directly from raw noisy data. The advantage of \eqref{eq:BSC_suff_cond} is that it does not involve any structure related information. On the other hand, it requires the knowledge of the noise parameters $q_i,q_j$ and the maximum correlation among the hidden nodes, or some accurate estimates of these parameters. Note that for identical noise $q_i=q_j=q$, it is true that $(1-2 q_i)/(1-2 q_j)=1$ for all $i,j\in\mc{V}$, thus \eqref{eq:BSC_suff_cond} is always satisfied because  $|\mbb{E}[X_i X_j]|\in (0,1)$, and structure learning is always feasible for this regime.

If condition \eqref{eq:BSC_suff_cond} is not satisfied then structure recovery is still feasible by applying an appropriate pre-processing on the data $\bnX^{1:n}$. The pre-processing procedure requires the values $q_i$ (or estimates of them) to be known. By considering the pre-processing $Z_i=Y_i/(1-2q_i)$ on the input data of Algorithm \ref{alg:Chow-Liu}, the new weights are $W_{i,j}=\hat{I}(Z_i;Z_j)$. The latter enforces an appropriate processing (see Definition \ref{ordering_processed}) and guarantees that the algorithm is consistent, \begin{align}
     \P\lp\lim_{n\to \infty}\mc{E}_{\TCL_{\dagger}(\mc{D}=\mbf{Z}^{1:n})}= \mc{E}_{\T}\rp=1.
\end{align} 
  Simulations on synthetic data verify our analysis (see Figure \ref{fig:negative_Ith, simulations}). 

\subsubsection{$M$-ary Erasure Channel with Non-Identically Distributed Noise}\label{M-ary_arasure} Assume that the randomized mappings $F_i(\cdot)$ ``erase'' each variable independently with probability $q$, so for all $i\in [p]$, we have $Y_i = X_i$ with probability $1 - q$ and $Y_i = M+1$ (an erasure) with probability $q$. Then $I(Y_i;Y_j)=(1-q)^2I\lp X_i;X_j \rp$ for all $i,j\in\mc{V}$ (see Appendix \ref{M-ary_appendix}) and $\Ith_{\dagger}=(1-q)^2 \Ith\leq \Ith$. The latter guarantees that if $\Ith>0$ then $\Ith_{\dagger}>0$. Given the values of $p,\delta,q$ and $\Ith$, Theorem \ref{thm:sufficient_noisy} provides the sample complexity for exact structure recovery from noisy observations. For fixed values of $p$ and $\delta$, the ratio of sufficient number of samples in the noiseless and noisy settings is $\mc{O}\big((1-q)^{4(1+\zeta)}\big)$, for any $\zeta>0$. 

In contrast, consider the scenario where the erasure probability is not the same for every node (non-identically distributed noise). Each $F_i$ erases the $i^\text{th}$ node value with probability $q_i\in [0,1)$, so $I(Y_i;Y_j)=(1-q_i)(1-q_j) I\lp X_i;X_j \rp$ for all $i,j\in\mc{V}$, and the condition $\Ith_{\dagger}>0$ shows that Algorithm \ref{alg:Chow-Liu} with input $\mc{D}=\bnX^{1:n}$ converges; $\TCLn \to\T$, if for all tuples $(w,w),u,\bar{u}\in\fset$ \begin{align}\label{eq:ex2_cond}
    \frac{(1-q_w)(1-q_{\bar{w}})}{(1-q_u)(1-q_{\bar{u}})}
    >\frac{I(X_u;X_{\bar{u}})}{I(X_w;X_{\bar{w}})}.
\end{align} 
Define $\mbf{RI}\triangleq \max_{(w,\bar{w}),u,\bar{u}\in\fset}I(X_u;X_{\bar{u}})/I(X_w;X_{\bar{w}})$. Inequality \eqref{eq:ex2_cond} provides the following simplified sufficient condition for convergence of CL Algorithm with input noisy data; $\TCLn\to\T$ if for all $i,j\in\mc{V}$
    \begin{align}\label{eq:simplified_test}
    &\frac{1-q_i}{1-q_j}\in \lp \mbf{RI}^{1/2},\mbf{RI}^{-1/2}\rp.
    \end{align}

\noindent However, \eqref{eq:simplified_test} is not useful because the quantity $\mbf{RI}$ involves knowledge related to the unknown tree structure. Thus is not possible to check if \eqref{eq:simplified_test} holds before running the Chow-Liu algorithm. For that reason we would like to derive a relaxed version. To find a relaxed condition, we restrict the class of interest by introducing the following assumptions on the hidden layer $\bX$:
\begin{itemize}
    \item For a fixed known value $\Ith_{\min}$, it is true that $\Ith\geq\Ith_{\min}>0$.
    \item For a fixed known value $I_{\max}$ the inequality $ I_{\max} \geq I(X_i;X_j)$ holds for all $(i,j)\in\mc{E}$.
\end{itemize} Then condition \eqref{eq:simplified_test} can be relaxed as 
\begin{align}
    \frac{1-q_i}{1-q_j}\in \lp \sqrt{\frac{I_{\max}-2\Ith_{\min}}{I_{\max}}}, \sqrt{\frac{I_{\max}}{I_{\max}-2\Ith_{\min}}}  \rp. \label{eq:simplified_test2}
\end{align} Given the values $q_i,q_j$ and $\Ith_{\min}$, $I_{\max}$ or accurate estimates of them, \eqref{eq:simplified_test2} provides a rule for testing if structure estimation is possible directly from raw noisy data. Additionally, the above assumption on $\Ith$ and $I(X_i,X_j)$ for edge pairs $(X_i,X_j)$ is equivalent with an assumption on strong and weak edges of the tree, which is common in the literature (for instance see related work by~\cite{bresler2020learning}.)

If direct structure estimation is not possible then we can consider an appropriate processing by considering as edge weights the values $W_{i,j}=\hat{I}(Y_i;Y_j)/(1-q_i)(1-q_j)$. Under the new input $\mc{W}$ of the MST, the algorithm becomes consistent by satisfying the property of the appropriate processing in Definition \ref{ordering_processed}. On the other hand, for identical distributed noise for each observable node ($q_i=q$, for all $i\in\mc{V}$), the condition \eqref{eq:simplified_test2} always holds, thus structure recovery is always possible for identically distributed noise. Lastly, a comparison of the required processing for $M-ary$ erasure channels and binary symmetric channels with non-identically distributed noise (Section \ref{BSC_example}), indicates that different processing procedures are required for different models.

\begin{figure*}[t!]
\centering
\includegraphics[width=215pt, valign=t]{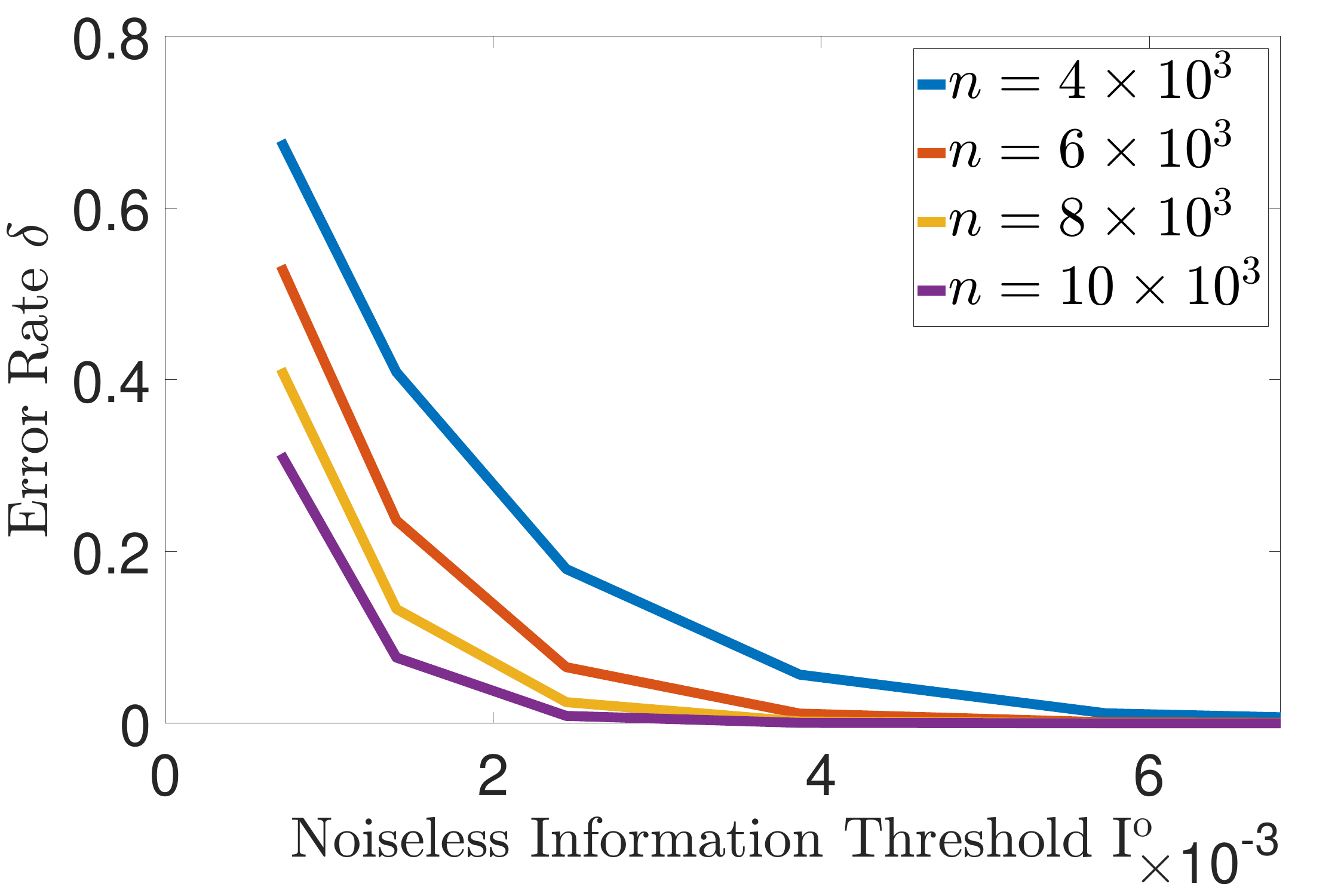}
 \includegraphics[width=215pt,valign=t]{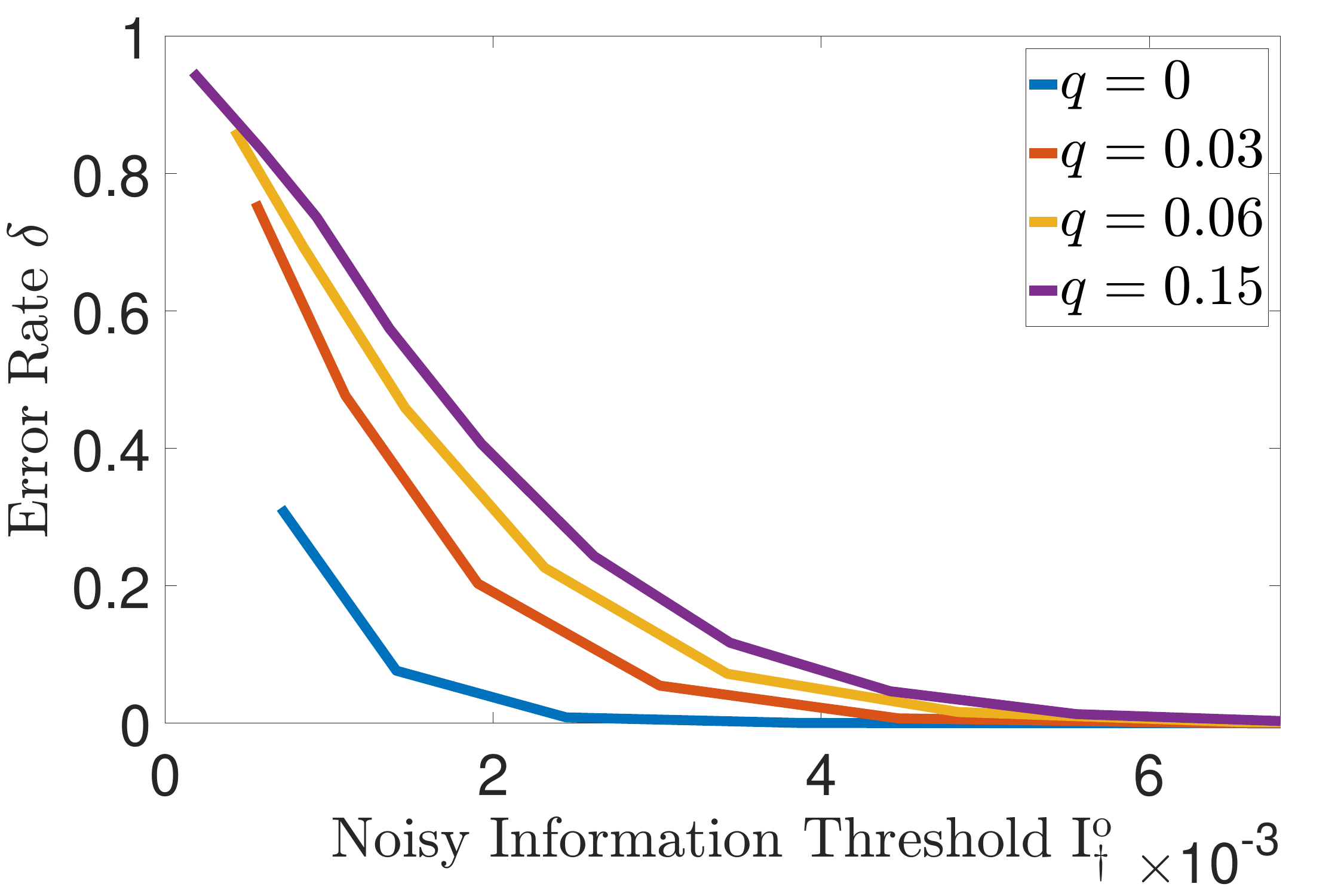}
 \includegraphics[width=215pt,valign=t]{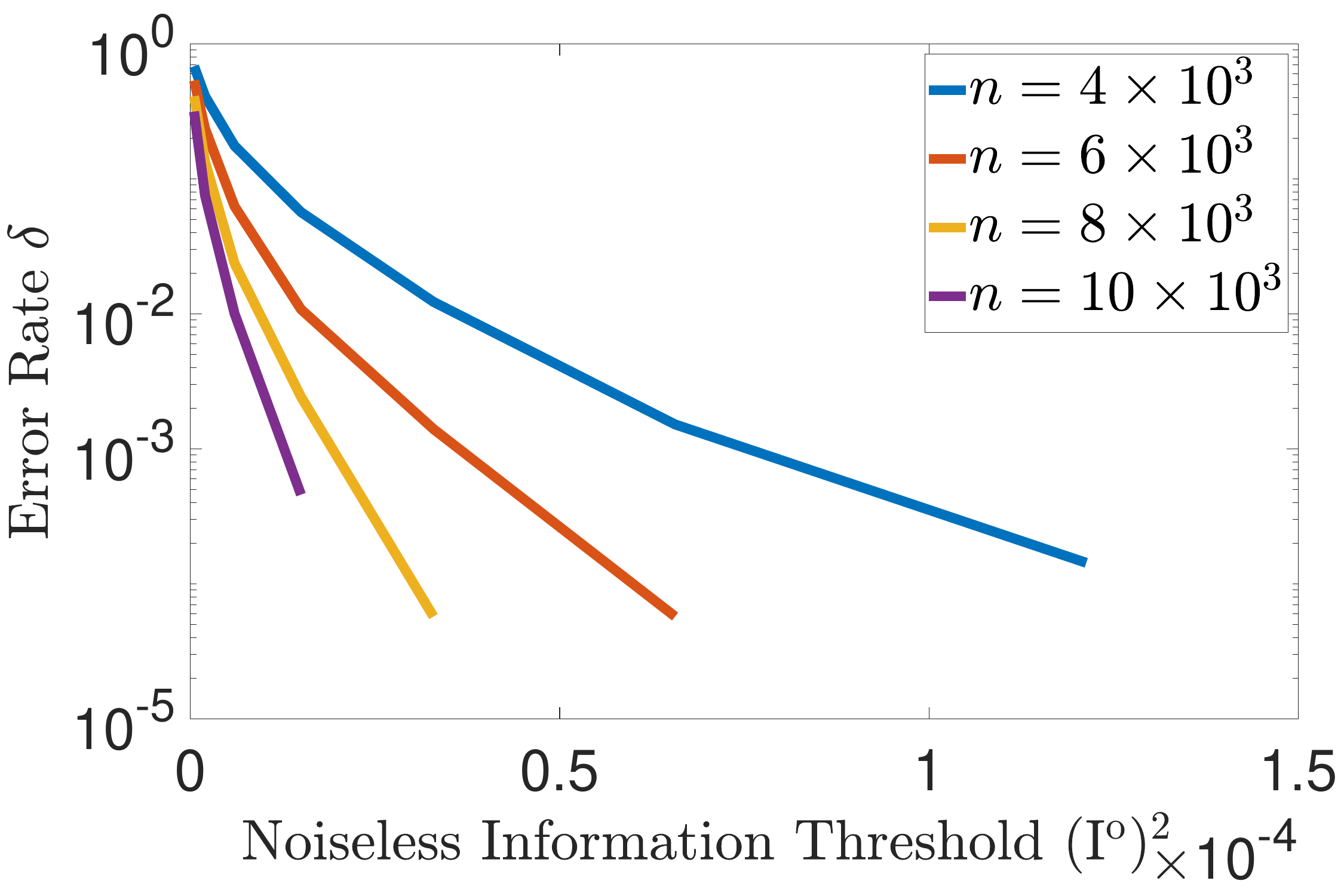}
 \includegraphics[width=215pt,valign=t]{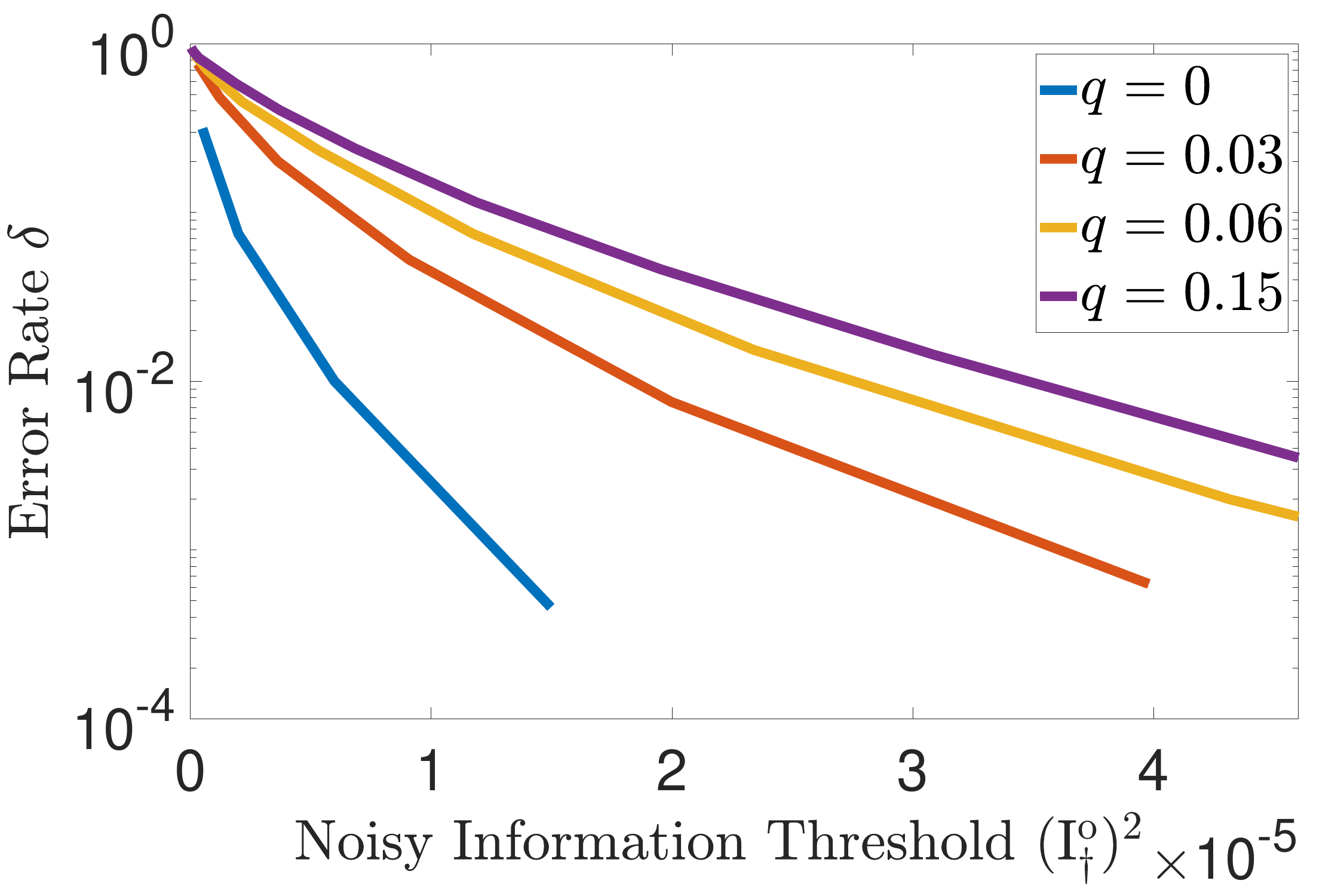}
\caption{Left: estimating the probability $\P\lp\TCL\neq \T \rp$ for different values of $\Ith$ and $n$ through $3.5\times10^4$ independent runs with noiseless data. Right: estimating the probability $\P\big(\TCL_{\dagger}\neq \T \big)$ for different values of $\Ith_{\dagger}$ and $q$, $n=10^4$, samples and $3.5\times10^4$ independent runs with noisy data. The noisy information threshold $\Ith_{\dagger}$ depends on moments of the hidden layer, as well as on the value $q$. As a consequence, for different values of $q$ the points of x-axis are slightly different. Further, the case $q=0$ gives the noiseless setting $\Ith_{\dagger}\equiv \Ith$. The code is available at \url{https://github.com/KonstantinosNikolakakis/Information-Threshold}.  \label{fig:f1}}
\end{figure*}

\section{Simulations}
In this Section, we provide simulations to illustrate the dependence between the error rate and the corresponding value of the information threshold. For the simulations of this paper, the tree-structured of the hidden layer is generated randomly. Starting from the root, we choose the parent of each new node uniformly at random among the nodes that are currently in the tree, in a sequential fashion. To demonstrate the relationships between $\delta$, $n$ and $\Ith$, $\Ith_{\dagger}$,  we estimate $\TCL,\TCL_{\dagger}$ and $\delta$ through $3.5\times10^4$ independent runs on tree-structured synthetic data, for different values of $n\in [10^3 ,10^4 ]$ and $p=10$. The variable nodes are binary and take values in the set $\{-1,+1\}$. Figure \ref{fig:f1} (left) illustrates the relationship between the probability of incorrect reconstruction and $\Ith$. 

Lastly, Figure \ref{fig:f1} (right) presents the effect of noise of a BSC for different values of the crossover $q$ and number of samples $n=10^4$. Notice that the probability of incorrect reconstruction also decays exponentially with respect to the noisy information threshold, as Theorem \ref{thm:sufficient_noisy} suggests, but with a significantly smaller rate than the noiseless case ($q=0$). To derive different values of $\Ith_{\dagger}$ for the purpose of the experiment, we consider a range for the correlations between the hidden variables, while the parameter of the noise is considered fixed ($q= 0, 0.03, 0.06, 0.15$). We observe that probability of incorrect reconstruction decays exponentially with respect to the noisy information threshold, as Theorem \ref{thm:sufficient_noisy} suggests.

\section{Conclusion \& Future Directions}\label{coclusion}

In this paper we showed how the \emph{information threshold} characterizes the problem of recovering the structure of hidden tree-shaped graphical models. This quantity arises naturally in the error analysis of Chow-Liu algorithm. As our main contribution, we introduced the first finite sample complexity bound on the performance of the CL algorithm for learning tree structured models from noisy data, while the alphabet size is up to countable, while the models are general. More specifically, the sufficient number of samples bound shows how the number of nodes $p$, the probability of failure $\delta$, and $\Ith_{\dagger}$ (the information threshold) are related for the problem of structure recovery. In fact, we provide matching rates (upper and lower sample complexity bounds) with respect to information threshold. Consequently, the CL algorithm achieves an optimal rate with respect to the statistic $\Ith_{\dagger}$. Our results also demonstrate how noise affects the sample complexity of learning for a variety of standard models, including models for which the noise is not identically distributed. 

Although we strictly consider the class of tree-structured models in this paper, our approaches of Theorem \ref{thm:sufficient_noisy} and Theorem \ref{thm:converse_noisy} can be extended to the class of forests. For that purpose, we should consider a generalization of $\Ith$ to forests and a modified version of CL the CLThres algorithm~\citep{tan2011learning}. We leave this part for future as it is out of the scope of this paper.

Additionally, our approach is more generally applicable to the analysis of $\delta$-PAC Maximum Spanning Tree (MST) algorithms. At its root, our work shows how the error probability of MST algorithms (for example, Kruskal's algorithm or Prim's algorithm) behaves when edge weights are uncertain, i.e., when only (random) estimates of the true edge weights are known.

To conclude, the non-parametric graphical model setting presents interesting theoretical challenges that are connected with other statistical problems, out of the focus of this paper. The relationship between $\Ith$ and $\Ith_{\dagger}$ is connected with open problems in information theory related to Strong Data Processing Inequalities~\citep{raginsky2016strong,polyanskiy2017strong}, for which tight characterizations are only known for a few channels. In our situation, a general analytical relationship may be similarly challenging. From a practical standpoint, we may wish to estimate the sample size needed to guarantee recovery with a pre-specified error probability. Doing so would require knowing $\Ith_{\dagger}$ before collecting the full data; since $\Ith_{\dagger}$ depends on the noise \emph{model}, we could find such a bound by considering a reasonable class of underlying models and taking the worst case. An interesting open question for future work is how to effectively estimate $\Ith_{\dagger}$ from (auxiliary) training data rather than relying on such a priori modeling assumptions. This may help design appropriate processing methods that can make structure learning algorithms more robust against noise or adversarial attacks.



\appendix


\section{Proofs \& Results }\label{Appendix}
We start by providing the proofs of Propositions \ref{prop1} and \ref{prop_locality}.
\subsection{Proof of Proposition \ref{prop1}}
We consider the case $u^*\equiv w^*$ and $\bar{u}\in\mc{N}_{\T}(\bar{w})$, while the other three cases that are given by the locality property can be identically proved. The case $u^*\equiv w^*$ and $\bar{u}\in\mc{N}_{\T}(\bar{w})$ implies that $w^* - \bar{w}^* -\bar{u}^*$ is a subgraph of $\T$. Assume for sake of contradiction that $I(X_{w^*};X_{\bar{w}^*})=I(X_{w^*};X_{\bar{u}^*})$ then $I(X_{w^*};X_{\bar{w}^*}|X_{\bar{u}^*})=0$. The latter implies that $w^*  -\bar{u}^* - \bar{w}^*$ is also a subgraph of $\T$ and it contradicts with the uniqueness of the structure (Assumption \ref{ass:unique_tree}).\QEDB

\subsection{Proof of Proposition \ref{prop_locality}}

Assume for sake of contradiction that $u^*\neq w^*$ and $u^*\neq \tilde{w}^*$ or $\bar{u}\notin\mc{N}_{\T}(w)$ and $\bar{u}\notin\mc{N}_{\T}(\bar{w})$ and let $\nu$ be a node such that $\nu\in\mc{N}_{\T}(w)\cup\mc{N}_{\T}(\bar{w})$, then the data processing inequality~\citep{cover2012elements} and Assumption \ref{ass:unique_tree} give \begin{align}
   I(X_{w};X_{\bar{w}})-I(X_{\bar{w}};X_{\nu})  < I(X_{w};X_{\bar{w}})-I(X_{u};X_{\bar{u}})
\end{align} and \begin{align}
    I(X_{w};X_{\bar{w}})-I(X_{w};X_{\nu})  < I(X_{w};X_{\bar{w}})-I(X_{u};X_{\bar{u}}).
\end{align} The last two inequalities contradict the assumption \eqref{eq:argmin}.\QEDB

\section{Fano's Inequality}\label{Appendix_converse}

\begin{lemma}[\textbf{Fano's Inequality, \citep{tsybakov2009introduction}}]\label{Fano's_method}
  Fix $M\geq 2$ and let $\Theta$ be a family of models $\theta^0,\theta^1,\ldots,\theta^M$. Let $\P_{\theta^j}$ denote the probability law of $\bX$ under model $\theta^j$, and consider $n$ i.i.d. observations $\bX^{1:n}$.  If \begin{align}
    n< (1-\delta) \frac{\log M}{\frac{1}{M+1}\sum^{M}_{j=1}  \KL(\P_{\theta^j}||\P_{\theta^0})}, 
\end{align} then it is true that\begin{align}
    \inf_{\Phi}\max_{0\leq j\leq M} \P_{\theta^j} \left[ \Phi (\bX^{1:n}) \neq j \right]\geq \delta -\frac{1}{\log (M)},
\end{align}
where the infimum is relative to all estimators (statistical tests) $\Phi:\mc{X}^{p\times n}\to \{0,1,\ldots,M\}$.
\end{lemma}

\subsection{$M$-ary Erasure Channel}\label{M-ary_appendix}
\noindent For the $M$-ary erasure channel, it is true that
    \begin{align}
        I(Y_i;Y_j)=(1-q_i)(1-q_j)I\lp X_i;X_j \rp.
    \end{align} for all $i,j\in\mc{V}$ and $q_i,q_j\in[0,1)$. To prove this, we start by expanding the mutual information from the definition and pulling out the erasure event as follows
\begin{align*}
    &I(Y_i;Y_j)\\
    &=\sum_{y_i,y_j\in[M+1]^2}          
        \p_{\dagger}(y_i,y_j)\log\frac{\p_{\dagger}(y_i,y_j)}{\p_{\dagger}(y_i)\p_{\dagger}(y_j)}\\
    &=\sum_{y_i,y_j\in[M]^2}
        \p_{\dagger}(y_i,y_j)\log\frac{\p_{\dagger}(y_i,y_j)}{\p_{\dagger}(y_i)\p_{\dagger}(y_j)}
    +\sum_{y_i\in[M]}\p_{\dagger}(y_i,M+1)\log\frac{\p_{\dagger}(y_i,M+1)}{\p_{\dagger}(y_i)\p_{\dagger}(M+1)}
    \\
    &\quad +\sum_{y_j\in[M]}\p_{\dagger}(M+1,y_j)\log\frac{\p_{\dagger}(M+1,y_j)}{\p_{\dagger}(M+1)\p_{\dagger}(y_j)}+\p_{\dagger}(M+1,M+1)\log\frac{\p_{\dagger}(M+1,M+1)}{\p_{\dagger}(M+1)\p_{\dagger}(M+1)}\\
    &=\sum_{y_i,y_j\in[M]^2} \p_{\dagger}(y_i,y_j)\log\frac{\p_{\dagger}(y_i,y_j)}{\p_{\dagger}(y_i)\p_{\dagger}(y_j)}\\
    &=\sum_{x_i,x_j\in[M]^2}(1-q_i)(1-q_j)\p(x_i,x_j)\log\frac{\p(x_i,x_j)}{\p(x_i)\p(x_j)}\\
    &=(1-q_i)(1-q_j)I\lp X_i;X_j \rp.\numberthis \label{eq:erasure_proof_1}
\end{align*} 
An erasure occurs independently on each node variable observable and independently with respect to the $\bX$, thus  $\p_{\dagger}(y_i,M+1)=\p_{\dagger}(y_i)\p_{\dagger}(M+1)$, for any $y_i\in[M+1]$ and  $\p_{\dagger}(M+1,y_j)=\p_{\dagger}(M+1)\p_{\dagger}(y_j)$ for any $y_j\in[M+1]$. The latter gives \eqref{eq:erasure_proof_1}. \QEDB

\subsection{Binary Symmetric Channel with Non-Identically Distributed Noise}\label{BSC_non_id_noise}

\noindent Under the assumption of $\Ith>0$, we wish to show that if \begin{align}\label{eq:sufficient}
   \!\!\!\! \frac{(1-2q_{i})}{(1-2q_{j})}\in  \lp \max_{(i,j)\in\mc{E}_{\T}} |\mbb{E}[X_i X_j]|, \frac{1}{\underset{(i,j)\in\mc{E}_{\T}}{\max} |\mbb{E}[X_i X_j]|}\rp, 
\end{align} for all $i,j\in\mc{V}$ then $\Ith_{\dagger}>0$. We start by finding the values of the sequence of crossover probabilities $q_1,q_2,\ldots,q_k\in[0,1/2)$ which guarantee that $\Ith_{\dagger}>0$. The mutual information of two binary random variables $Y_i,Y_j\in\{-1,+1\}$ (see~\citep{nikolakakis2021predictive}) is
    \begin{align}\label{eq:MI_binary}
	&I\left(Y_{i},Y_{j}\right)\\
	&=\frac{1}{2} \log_{2} \lp \lp 1-\E\left[Y_{i}Y_{j}\right]\rp^{1-\E\left[Y_{i}Y_{j}\right]} \lp 1+\E\left[Y_{i}Y_{j}\right]\rp^{1+\E\left[Y_{i}Y_{j}\right]}\rp.\nonumber
	\end{align} 

\noindent The definition of $\Ith_{\dagger}$ (Definition \ref{def:Ith_noisy}) and \eqref{eq:MI_binary} give \begin{align}
&\Ith_{\dagger}= \frac{1}{2}  \left\{I\lp Y_w;Y_{\bar{w}} \rp - I\lp Y_u;Y_{\bar{u}} \rp\right\}\label{eq;MI_BSC}\\
&\!\!=\frac{1}{2}  \log_{2} \! \frac{\lp 1-\E\left[ Y_w Y_{\bar{w}}  \right]\rp^{1-\E\left[Y_w Y_{\bar{w}}\right]} \lp 1+\E\left[Y_w Y_{\bar{w}}\right]\rp^{1+\E\left[Y_w Y_{\bar{w}}\right]}}{\lp 1-\E\left[Y_u Y_{\bar{u}}\right]\rp^{1-\E\left[Y_u Y_{\bar{u}}\right]} \lp 1+\E\left[Y_u Y_{\bar{u}}\right]\rp^{1+\E\left[Y_u Y_{\bar{u}}\right]}}.\nonumber
\end{align}

\noindent Define the function $f(\cdot)$ as
    \begin{align}
    f(x)\triangleq \lp 1-x\rp^{1-x} \lp 1+x\rp^{1+x}\equiv f(|x|),
    \end{align} 
then 
    \begin{align}
   &\Ith_{\dagger}= \frac{1}{2}  \log_{2} \frac{f(\left|\E\left[ Y_w Y_{\bar{w}}  \right]\right|)}{f(|\E\left[ Y_u Y_{\bar{u}}  \right]|)}
    \end{align} and \begin{align}
       \E\left[ Y_u Y_{\bar{u}}  \right] = (1-2q_w)(1-2q_{\bar{w}}) \E\left[ X_w X_{\bar{w}}  \right],
    \end{align}\begin{align}
       \E\left[ Y_w Y_{\bar{w}}  \right]&= (1-2q_w)(1-2q_{\bar{u}}) \E\left[ X_w X_{\bar{w}}  \right]\times\prod_{(i,j)\in\tpath_{\T}(u.\bar{u})\setminus(w,\bar{w})}\E\left[ X_i X_j  \right]
    \end{align}
for the last equality we used the correlation decay property~\citep{bresler2020learning,nikolakakis2019learning,nikolakakis2021predictive}) and the fact that for $\pm 1$-valued variables the binary symmetric channel can be consider as multiplicative binary noise~\citep{nikolakakis2019learning,nikolakakis2021predictive}. 
Note that $f(x)$ is increasing for $x>0$. To guarantee that $ \Ith_{\dagger}>0$ we need 
    \begin{align}
   \!\!\!\! \frac{(1-2q_w)(1-2q_{\bar{w}})}{(1-2q_u)(1-2q_{\bar{u}})}>\!\!\!\!\prod_{(i,j)\in\tpath_{\T}(u.\bar{u})\setminus(w,\bar{w})}|\E\left[ X_i X_j  \right]|.\label{eq:condition_on_qi}
\end{align}
Recall that \eqref{eq:condition_on_qi} should hold for all $(w,\bar{w})\in\mc{E}$ and for all $u,\bar{u}\in\mc{V}$ such that $(w,\bar{w})\in\tpath_{\T}(u,\bar{u})$. 

\noindent In addition, \begin{align*}
   \prod_{(i,j)\in\tpath_{\T}(u.\bar{u})\setminus(w,\bar{w})}|\E\left[ X_i X_j  \right]|\leq \max_{(i,j)\in\mc{V}} |\mbb{E}[X_i X_j]|.
\end{align*} The last two inequalities give the sufficient condition\begin{align*}
    \frac{(1-2q_w)(1-2q_{\bar{w}})}{(1-2q_u)(1-2q_{\bar{u}})}&>\max_{(i,j)\in\mc{V}} |\mbb{E}[X_i X_j]|.
\end{align*}

\noindent As a consequence, if  \begin{align*}
   \!\!\!\! \frac{(1-2q_{i})}{(1-2q_{j})}\in  \lp \sqrt{\max_{(i,j)\in\mc{V}} |\mbb{E}[X_i X_j]|}, \frac{1}{\sqrt{\underset{(i,j)\in\mc{V}}{\max} |\mbb{E}[X_i X_j]|}}\rp, 
\end{align*} for all $i,j\in\mc{V}$ then $\Ith_{\dagger}>0$. Note that for the case of i.i.d. noise $(q_{\bar{i}}=q_{\bar{j}}$ for all $i,j\in\mc{V}$) the inequality always holds because $\max_{(i,j)\in\mc{V}} |\mbb{E}[X_i X_j]|\in (0,1)$.\QEDB

\clearpage
\bibliography{Structure_learning_merged}
\end{document}